\documentclass{article}

\usepackage[final]{neurips_2025}

\usepackage{amsmath,amsthm,verbatim,amssymb,amsfonts,amscd, graphicx, enumitem}
\usepackage{graphics}
\usepackage{centernot}
\usepackage{authblk}

\theoremstyle{plain}
\newtheorem{theorem}{Theorem}

\newtheorem{lemma}{Lemma}

\newtheorem{definition}{Definition}

\usepackage[colorlinks,linkcolor=blue,citecolor=blue]{hyperref}

\usepackage[bottom]{footmisc}
\usepackage{caption}
\usepackage{subcaption}
\usepackage{helvet}  
\usepackage{courier}  
\usepackage{url}  
\usepackage{graphicx}  
\usepackage{multirow}
\usepackage{amsthm}
\usepackage{color}
\usepackage{MnSymbol}
\usepackage{makecell}
\usepackage{arydshln}
\usepackage{amsmath}
\usepackage[dvipsnames]{xcolor}
\usepackage{caption} 
\usepackage{natbib}
\usepackage{bbm}

\usepackage{textcomp}
\usepackage{wrapfig}
\usepackage{algorithm}
\usepackage{algorithmic}

\usepackage{csquotes}

\newcommand{\E}{\mathbb{E}}

\newcommand{\Tr}{{\rm Tr}}

\title{Neural Thermodynamics: Entropic Forces in Deep and Universal Representation Learning}

\begin{document}

\author{Liu Ziyin$^{1,3,*}$, Yizhou Xu$^{2,*}$, Isaac Chuang$^1$\\
$^1$\textit{Massachusetts Institute of Technology}\\
$^2$\textit{École Polytechnique Fédérale de Lausanne}\\
$^3$\textit{NTT Research}
}
\maketitle
\def\thefootnote{*}\footnotetext{Equal contribution.}\def\thefootnote{\arabic{footnote}}

\begin{abstract}
With the rapid discovery of emergent phenomena in deep learning and large language models, understanding their cause has become an urgent need. Here, we propose a rigorous entropic-force theory for understanding the learning dynamics of neural networks trained with stochastic gradient descent (SGD) and its variants. Building on the theory of parameter symmetries and an entropic loss landscape, we show that representation learning is crucially governed by emergent entropic forces arising from stochasticity and discrete-time updates. These forces systematically break continuous parameter symmetries and preserve discrete ones, leading to a series of gradient balance phenomena that resemble the equipartition property of thermal systems. These phenomena, in turn, (a) explain the universal alignment of neural representations between AI models and lead to a proof of the Platonic Representation Hypothesis, and (b) reconcile the seemingly contradictory observations of sharpness- and flatness-seeking behavior of deep learning optimization. Our theory and experiments demonstrate that a combination of entropic forces and symmetry breaking is key to understanding emergent phenomena in deep learning.
\end{abstract}

\vspace{-2mm}
\section{Introduction}
\vspace{-1mm}

Modern neural networks trained with stochastic gradient descent (SGD) exhibit a complex plethora of emergent behaviors -- emergence of capabilities \cite{wei2022emergent,michaud2023quantization,arora2023theory}, progressive sharpening and flattening \cite{cohen2021gradient, cohen2024understanding}, phase-transition like behaviors \cite{ziyin2023zeroth, ziyin2024symmetry}, and universal representational alignment across models \cite{huh2024platonic,yu2024representation,tjandrasuwita2025understanding,shu2025large} -- that are difficult to explain through loss minimization alone. These behaviors mirror phenomena found in physical systems at finite temperature, suggesting that deep learning dynamics are shaped not just by explicit optimization but also by implicit forces arising from stochasticity and discrete updates. These implicit forces have long been associated with the phenomenon of ``implicit bias'' in deep learning \cite{arora2019implicit, galanti2022sgd, damian2022self,xie2024implicit}, but their precise mathematical nature remains elusive. In physics, such effects are often captured by \textit{entropic forces}—macroscopic forces that emerge from the system’s statistical tendencies rather than its energy landscape alone \cite{cai2010friedmann}. The power of this framework lies in the notion of an \textit{effective entropy}, which plays the role of a potential whose gradients define the entropic force. Identifying this effective entropy not only reveals what the system is implicitly optimizing, but also opens the door to leveraging theoretical tools from statistical physics to analyze and improve AI models.

\paragraph{Contributions.} We formalize this connection between stochastic learning dynamics and entropic forces through the lens of symmetry and representation learning to:
\begin{enumerate}[noitemsep,topsep=0pt, parsep=0pt,partopsep=0pt, leftmargin=15pt]
    \item Derive an entropic loss function and show that the entropic forces of SGD \textit{break continuous parameter symmetries while preserving discrete ones} (Section~\ref{sec: entropic loss}).
    \item Show that the symmetry breaking due to entropic forces gives rise to a family of \textit{equipartition theorems} that predict the gradient alignment phenomena (Section~\ref{sec: equipartition}).
    \item Explain and unify two seemingly disparate but universal observations -- \textit{progressive sharpening of the loss landscape} and the \textit{emergence of universal representations} -- as consequences of entropic forces (Section~\ref{sec: universal}).
\end{enumerate}

Our theory establishes a principled framework -- akin to a thermodynamics of deep learning -- that unifies several universal phenomena under a single formalism. The results suggest that the entropic loss landscape, shaped by both optimization and entropy, plays a foundational role in understanding learning dynamics and emergent phenomena. Full derivations and experimental validations are provided in the appendix.

\vspace{-2mm}
\section{Related Work}
\vspace{-1mm}
\paragraph{Modified Loss and Effective Landscape.}
The concept of modified or effective losses has emerged as a critical framework for understanding the implicit biases induced by stochastic gradient descent (SGD) in deep learning, which differs from another line of work \citep{yaida2018fluctuation,zhang2019algorithmic,ziyin2022strength,liu2021noise,ziyin2025noise} which leverages the property of stationarity to analyze the stationary distribution of SGD. Ref.~\citep{barrett2020implicit} introduced the notion of a modified loss to analyze the discrete-time dynamics of SGD, demonstrating how discretization implicitly alters the optimization landscape. Similarly,  Refs.~\citep{smith2021origin} and \cite{dandi2022implicit} extended the modified loss formulation to where there is a gradient noise due to minibatch sampling. These works conducted numerical simulations to show that training on the effective loss really approximates the original dynamics \cite{smith2021origin, geiping2021stochastic,kunin2023limiting} and leads to similar generalization performances. In this work, we refer to this type of losses as entropic losses for their associations with theoretical physics. Still, these entropic losses remain poorly understood, and their significance for understanding emergent phenomena in deep learning is not yet appreciated. Our work finds the crucial link between the entropic loss and symmetry-breaking dynamics, which is important for understanding the various intriguing nonlinear phenomena of representation learning. 

\vspace{-2mm}
\paragraph{Parameter Symmetry in Neural Networks.}
Parameter symmetries are shown to play a fundamental role in shaping neural network training dynamics and their emergent properties \citep{li2016symmetry, zhao2022symmetry, ziyin2024parameter, ziyin2024symmetry, ziyin2024remove,ziyin2025parameter}. A series of works showed that continuous symmetries in the loss function give rise to conservation laws, which imply that the learning result of SGD training is strongly initialization-dependent \cite{hidenori2021noether, zhao2022symmetry, marcotte2023abide}. More recent works showed how any stochasticity or discretization effect could break the symmetries in a systematic way such that the learned solution is no longer dependent on the initialization, a hint of universality
\citep{chen2023stochastic,ziyin2024parameter,ziyin2024symmetry,ziyin2025parameter}. Particularly, Ref.~\citep{ziyin2024parameter} developed the formalism of exponential symmetries and proved that any loss function with an exponential symmetry leads to a symmetry-breaking dynamics that converges to unique fixed points. This point can be seen as the dynamical equivalence of our Theorem~\ref{theo: symmetry breaking}, which states that there is essentially no continuous symmetry in the entropic loss. In comparison, our framework takes a different perspective: we study the symmetry from a loss landscape perspective and identify these symmetry-breaking tendencies as entropic forces. This unified perspective enables us to understand the universal learning phenomena with a greatly simplified analysis. 

\vspace{-2mm}
\section{Effective Energy for Stochastic Gradient Learning}\label{sec: entropic loss}
\vspace{-1mm}
Define $\ell(x,\theta)$ to be the per-sample loss function. We can define the empirical risk as 
\begin{equation}
    L(\theta) = \E_\mathcal{B}[\E_{x\in\mathcal{B}}\ell_\gamma(x,\theta)],
\end{equation}
where $\ell_\gamma(x,\theta):=\ell(x,\theta)+\gamma||\theta||^2$, $\mathcal{B}$ represents the minibatch and $\gamma$ represents the weight decay. From a dynamical-system perspective, for an infinitesimal learning rate, the loss function coincides with the Bregman Lagrangian of this dynamics, and so one can leverage the Lagrangian formalism to understand the training of gradient flow \cite{wibisono2016variational}. This is particularly attractive from a theory perspective because modern theoretical physics are also founded on the Lagrangian formalism and this connection allows one to borrow physics intuitions to understand deep learning.

However, simply studying this loss function is insufficient to understand the learning dynamics of SGD at various learning rates $\eta$ due to the stochastic discrete-time nature of SGD. This motivates the definition of an entropic loss $\phi_\eta$ such that running $n$ steps of update on $\phi_\eta(\theta)$ with learning rate $\eta/n$ is the same as running one step of update on $\ell_\gamma$ for any $x$. Taking the limit $n \to \infty$, one can obtain a ``renormalized" loss function for which running gradient flow is the same as running gradient descent for the original loss. With this loss, it becomes possible again to leverage Lagrangian formalism to understand SGD training with discrete-time and stochastic learning. Because $\phi_0$ coincides with running gradient flow on $\ell_\gamma$, one must have that 
\begin{equation}
    \phi_\eta := \ell_\gamma + \eta \phi_1 + \eta^2 \phi_2 + O(\eta^3).
\end{equation}
We can also consider the more general case where the learning rate is a fixed symmetric matrix $\Lambda$ with $||\Lambda||=\eta$. The following theorem derives the entropic loss for this case. Many common algorithms, such as Adam, natural gradient descent, and even a wide range of biologically plausible learning rules \cite{ziyin2025heterosynaptic} can be seen as having matrix learning rate.

\begin{theorem}\label{theo: effective loss} (Entropic Loss)
For fixed $x$, starting from $\theta_0$ run one-step gradient descent with $\Lambda$ on $\ell_\gamma(x,\theta)$ to obtain $\theta_1$. Run $n-$step gradient descent with $\Lambda/n$ on $\phi_\Lambda(x,\theta):=\ell_\gamma(x,\theta)+\phi_{1\Lambda}(x,\theta)+\phi_{2\Lambda}(x,\theta)$ to obtain $\theta_n'$. Then, assuming $||\nabla^3\ell_\gamma(x,\theta)||\leq M$,
\begin{enumerate}[noitemsep,topsep=0pt, parsep=0pt,partopsep=0pt, leftmargin=15pt]
    \item if $\phi_{1\Lambda}(x,\theta) = \frac{1}{4}\nabla\ell_\gamma(x,\theta)^T\Lambda\nabla\ell_\gamma(x,\theta)$, then, $\theta_n'=\theta_1+O(||\Lambda||^3+||\Lambda||^2/n)$;
    \item moreover, if $\phi_{2\Lambda}(x,\theta)=\frac{1}{2}\nabla\ell_\gamma(x,\theta)^T\Lambda\nabla^2\ell(x,\theta)\Lambda\nabla\ell_\gamma(x,\theta)$, then $\theta_n'=\theta_1+O(||\Lambda||^4+||\Lambda||^2/n+||\Lambda||^3/n+||\Lambda||^3M)$.
\end{enumerate}
\end{theorem}

This $\phi_\eta$ also needs to hold in expectation with respect to the sampling of data points, and so one can define the expected entropic loss $F_{\eta,\gamma}(\theta) = \E[\phi_\eta(x, \theta)]$, 
which is, up to the first order in $\Lambda$ and $\gamma$:
\begin{align}
\label{eq: free energy}
F_{\eta,\gamma}(\theta)&=\underbrace{\mathbb{E}_x\ell(x,\theta)}_{\text{learning, symmetry}}+\underbrace{\gamma||\theta||^2}_{\text{regularization}} +\underbrace{\frac{1}{4}\mathbb{E}_\mathcal{B}||\sqrt{\Lambda}\mathbb{E}_{x\in\mathcal{B}}  \nabla \ell(x,\theta)||^2 }_{\text{effective entropy due to discretization error, noise, $:=S(\theta)$}} +  O(||\Lambda||^2).
\end{align}
Specializing to the first order and with a scalar learning rate, this equation reduces to the ``modified loss" previously derived in different contexts \cite{barrett2020implicit, smith2021origin, dandi2022implicit}. The derivation has a thermodynamic flavor as it essentially computes the degree to which the dynamics is irreversible, and the entropy term is the part that cannot be microscopically reversed. In Theorem~\ref{theo: effective loss}, the first-order term in $\Lambda$ encourages the model to have a small gradient fluctuation. The second-order term, on top of gradient regularization, encourages the model to move to flatter solutions; this term has been found to play a role in the edge of stability phenomenon \cite{cohen2024understanding}. However, since the first-order term is not yet well-understood, our work focuses on the first-order term in $\Lambda$. We also focus on a scalar learning rate $\Lambda =\eta I$ and will comment on the differences when the difference is essential. 

Now, treating the original loss plus regularization as an energy, the dynamics of gradient flow on $F$ contains an energy force and and an entropic force: $\dot{\theta} = -\eta (\nabla L + \gamma \theta + {\nabla S})$. The  $\nabla S$ term, the gradient of the effective entropy, will be called the ``entropic force." See Figure~\ref{fig:resnet} for an illustration of the entropic force and an example of how entropy evolves during training.





\begin{figure}[t!]
    \centering
    \includegraphics[width=0.59\linewidth]{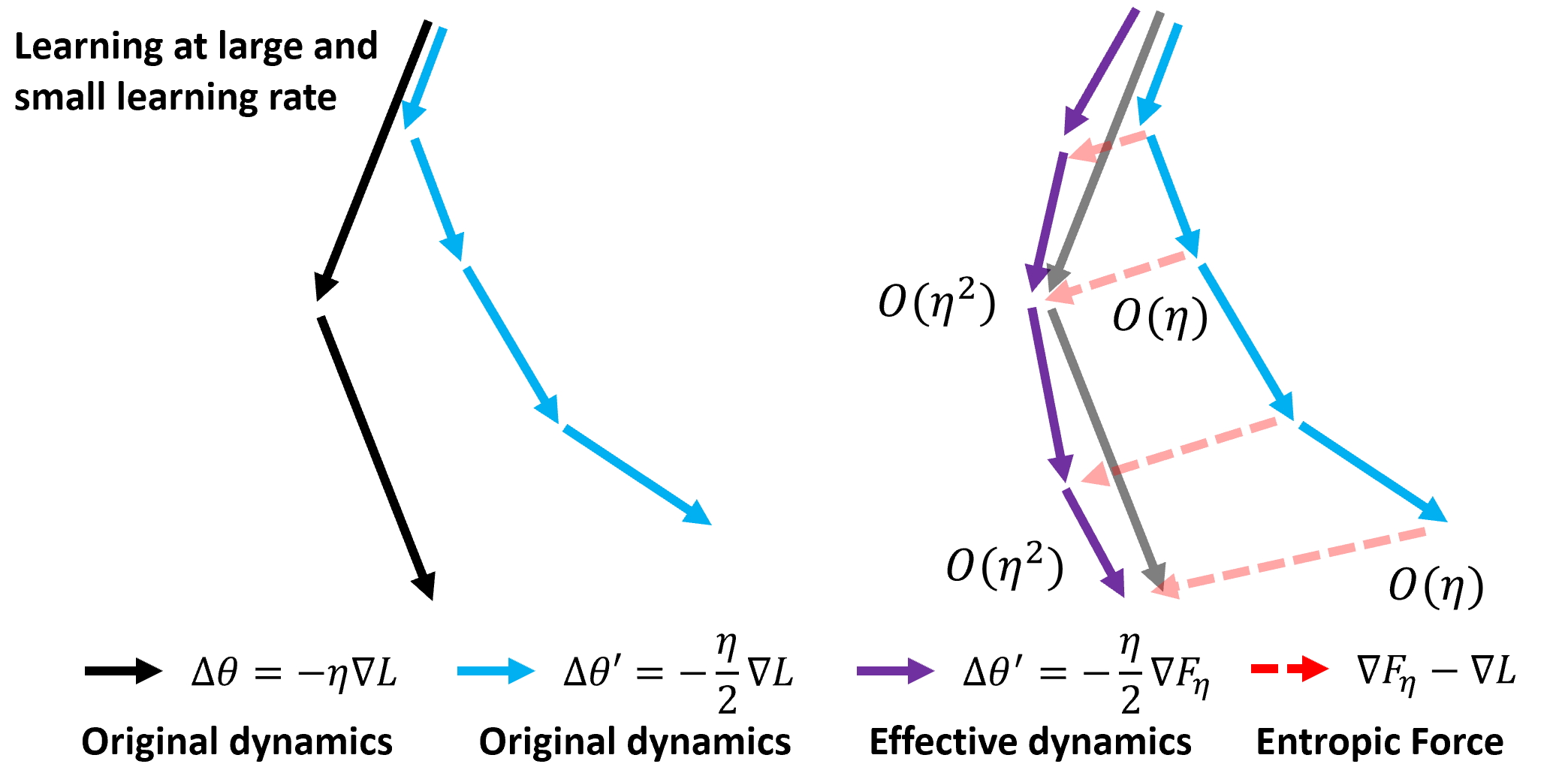}
    \includegraphics[width=0.39\linewidth]{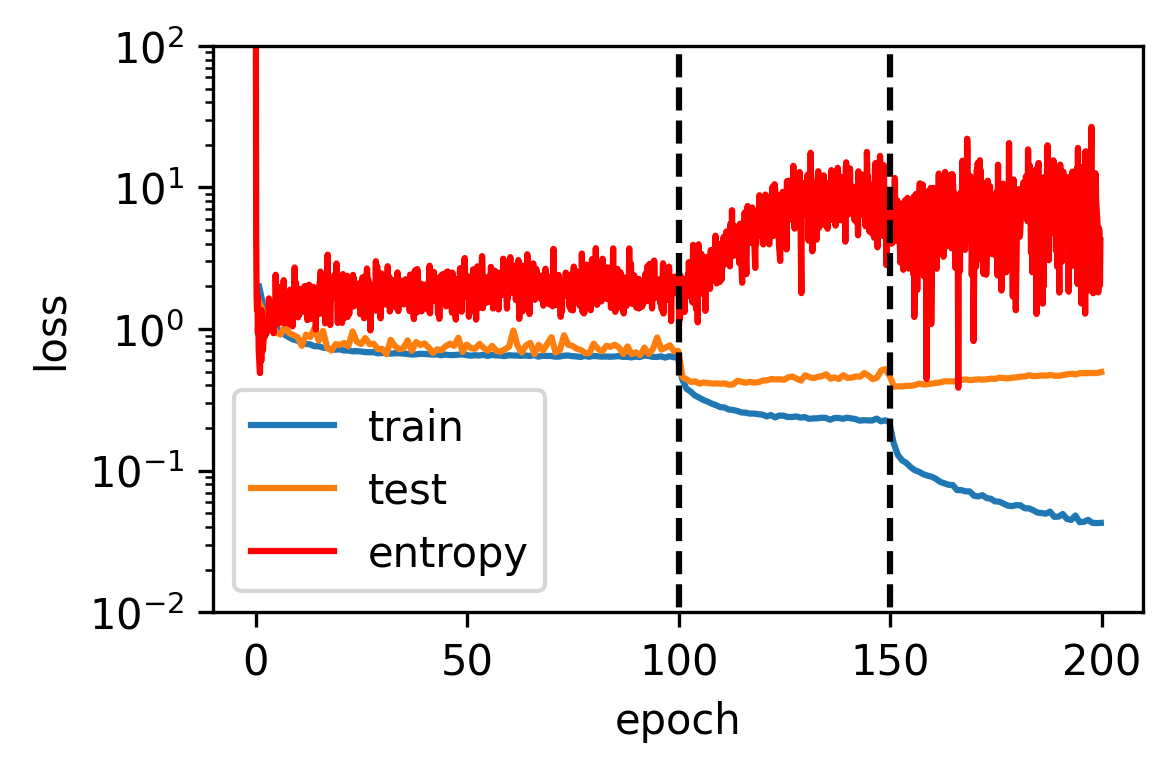}
    
    \caption{\small Entropic forces due to discretization error and stochasticity. \textbf{Left}: The learning dynamics of SGD at a large learning rate (LLR) and a small learning rate (SLR) is different. One can view the difference between LLR and SLR training as coming from an entropy term, which is an order $\eta$ force. After entropic correction, the difference between SLR and LLR is reduced to $O(\eta^2)$ and it becomes possible to analyze LLR SGD with gradient flow. \textbf{Right}: An example of entropic effect in neural network training. ResNet18 trained on CIFAR-10 with learning rate decay at $100$ and $150$ epochs. At the first learning rate drop (black dashed lines), the gradient (entropy) increases. This is unexpected and can only be explained by the entropic loss, where a large learning rate penalizes the entropy, and thus decreasing the learning rate leads to an increase in entropy. The second drop does not create too much effect because the learning rate is too small after the first drop.}
    \label{fig:resnet}
\end{figure}

We prove that the entropic force term breaks almost any continuous symmetry of $L$, a key result that we will leverage to study progressive sharpening and universal representation learning.
\begin{definition}
    A loss function $\ell(x, \theta)$ is said to be $K$-invariant if items 1-3 are satisfied:
     \begin{enumerate}[noitemsep,topsep=0pt, parsep=0pt,partopsep=0pt]
         \item locality: $K(\theta, \lambda ) = \theta + \lambda Q(\theta) + O(\lambda^2)$ for a differentiable $Q$;
         \item consistency: $K(K(\theta,\lambda), \lambda')= K(\theta, \lambda +\lambda')$;
         \item invariance: $\ell(x, K(\theta,\lambda)) = \ell(x, \theta)$ for all $x$, $\theta$ and $\lambda \in \mathbb{R}$.
         
     \end{enumerate}
     An entropic loss $F_{\eta,\gamma}$ is said to have the (\textbf{robust}) $K$-invariance if there exists a neighborhood around $\eta,\gamma$ such that $F_{\eta,\gamma}$ is $K$-invariant. 
\end{definition}

\begin{theorem}\label{theo: symmetry breaking} (Symmetry Breaking Under the Entropic Loss)
Let $\ell$ be $K$-invariant. If $F$ is also robustly $K$-invariant, then, (1) $\|K(\theta, \lambda)\| =\|\theta\|$ and (2) $\nabla^T\ell \nabla Q(\theta) \nabla \ell = 0$ for all $\theta$. 
\end{theorem}

This means that any symmetry or invariance that $F_{\eta,\gamma}$ has must be norm-preserving transformations. Essentially, this means that any invariance that is not rotation invariance must disappear. 
The following theorem shows that orthogonal discrete symmetries are preserved.

\begin{theorem}
\label{theo:orthogonal}
Let $OO^T = I$. If $\ell(x, O\theta) = \ell(x,\theta)$ for any $\theta$ and $x$, then $F_{\eta,\gamma}(O\theta) = F_{\eta,\gamma}(\theta)$ for any $\theta$.
\end{theorem}

Together with the previous result, this shows that when gradient noise or regularization is taken into account, the only relevant remaining symmetries are discrete symmetries. This implies that the results that are based on conservation laws for understanding SGD are questionable and can only hold in the toy setting of an infinitesimal learning rate. The reason is simple: $F_{\eta, \gamma}$ does not have \textit{robust} invariances at $\eta=\gamma =0$. Also, note that had we used a generic matrix learning rate (e.g., with Adam), the orthogonal invariances would also be broken. The meaning of these discrete symmetries can be understood through a framework similar to that proposed in Ref.~\cite{ziyin2024symmetry} and is left to a future work.




\vspace{-2mm}
\section{Emergence of Gradient Balance and Equipartition Property}\label{sec: equipartition}
\vspace{-1mm}

Lie group symmetries exist abundantly in nature and in modern neural networks\footnote{See Ref.~\cite{ziyin2025parameter} for a detailed review.} \cite{li2016symmetry, zhao2022symmetry, ziyin2024parameter, ziyin2024symmetry, hidenori2021noether}. In thermodynamics, the existence of symmetries is a crucial fact that leads to the emergence of hierarchical phenomena and phase transitions between them. In a sense, symmetry can be argued to be the ``first-order" approximation of the level of hierarchies in the system \cite{anderson1972more}. 
The following theorem states that the entropic loss $F$ breaks any nontrivial noncompact Lie group symmetries (also known as exponential symmetries \cite{ziyin2024parameter}). For formality, we say that $\ell$ (and $L$) has a $A$-\textbf{exponential symmetry} if for any $\lambda \in \mathbb{R}$, any $x$ and $\theta$, and any matrix $A$, $\ell(x,\theta)$ obeys $\ell(x,e^{\lambda A}\theta)=\ell(x,\theta)$. 
\begin{theorem}
\label{theo:exp symmetry}
(Master Balance Theorem) If $\ell(x,\theta)$ has an $A$-exponential symmetry,
then any local minimum $\theta^*$ of Eq.\eqref{eq: free energy} satisfies
\begin{equation}
-\eta\mathbb{E}_\mathcal{B
}[\mathbb{E}_{x\in\mathcal{B}}(\nabla_\theta\ell(x,\theta^*))]^T\tilde{A}[\mathbb{E}_{x\in\mathcal{B}}\nabla_\theta\ell(x,\theta^*)]+4\gamma(\theta^*)^T\tilde{A}\theta^*=0,
\label{eq:local_minimum}
\end{equation}
where $\tilde{A}:=\frac{A+A^T}{2}$. In addition, either (1) $F_{\eta,\gamma}(e^{\lambda A}\theta^*)=F_{\eta,\gamma}(\theta^*)$ for all $\lambda$, or (2) there exists no $\lambda \neq 0$ such that $e^{\lambda A}\theta^*$ is a local minimum.
\end{theorem}

If we had chosen a symmetric $A$, we would have $\tilde{A}=A$. For an anti-symmetric $A$, this theorem is trivial, consistent with Theorem~\ref{theo: symmetry breaking}. Therefore, this theorem is a statement about non-compact Lie group symmetries that extends to infinity. The fact that every point is connected to a point that satisfies Eq~\eqref{eq:local_minimum} means that SGD can reach this condition easily. The meaning of this theorem is a general gradient balance and alignment phenomenon. When $\gamma=0$, this equation states that the gradient along the positive spectrum of $\tilde{A}$ must balance with the gradient along the negative spectrum. When $\gamma \neq 0$, there is, additionally, a tradeoff between gradient balance and weight balance. We apply this result to various neural networks in this section.

Many works have shown that when training with weight decay or when the model has a small initialization, the weights of the layers become balanced, especially in homogeneous networks \cite{du2018algorithmic, polyakov2023homogeneous, rangamani2022neural}. Theorem~\ref{theo:exp symmetry} shows that the SGD in discrete-time or with stochasticity leads to a completely different bias where the gradient noise between all layers must be balanced. 

\paragraph{ReLU Layers} Consider a deep ReLU network trained on an arbitrary task:
\begin{equation}
    f(x) = W_{D} R_{D-1} ... R_1 W_1 x,
\end{equation}
where $R(x)$ is a piece-wise constant the zero-one activation matrix functioning as the ReLU activation. The entropy term can be written as
\begin{equation}
     S(\theta)= \sum_{i=1}^D\Tr \E [g_ig_i^\top],
\end{equation}
where $g_i=\mathbb{E}_{x\in\mathcal{B}}\nabla_{W_i}\ell(\theta,x)$ and $\mathbb{E}$ is a shorthand of $\mathbb{E}_\mathcal{B}$. Namely, we can group the gradient covariance according to layer index $i$. The following theorem states that all layers must have a balanced gradient. The proof shows that while the learning term $L$ is invariant to a class of symmetry transformations, the entropy term is not -- and this creates a systematic tendency for the parameters to reduce the entropy.

\begin{theorem}
\label{theo:layer_balance}
    (Layer Balance) For all local minimum of Eq.~\eqref{eq: free energy}, 
    \begin{equation}
         \eta(\mathbb{E}\Tr[g_ig_i^T-g_jg_j^T])=4\gamma(\Tr[W_iW_i^T-W_jW_j^T]).
    \end{equation}
    
\end{theorem}
Specifically, for $\gamma=0$ we have gradient balance $\E \Tr[g_ig_i^\top] =  \E \Tr[g_jg_j^\top]$. For $\eta=0$, we have the standard weight balance $\Tr[W_iW_i^T]=\Tr[W_jW_j^T]$. Otherwise, the solution interpolates between gradient balance and weight balance.

Similarly, within every two neighboring layers, one can group the parameters into neurons:
\begin{equation}
    \Tr[g_ig_i^\top] + \Tr[g_{i+1}g_{i+1}^\top] = \sum_j \left(\Tr[g_{i,j,:}g_{i,j,:}^\top] + \Tr[g_{i+1,:, j}g_{i+1,:,j}^\top] \right),
\end{equation}
where $g_{i,j,:}$ is the incoming weights to the $j$-the neuron of the $i$-th layer, and $g_{i+1,:,j}$ is the outgoing weights of the same neuron. The gradient for each neuron must also be balanced because there is a rescaling symmetry in every neuron.

\begin{theorem}\label{theo: neuron balance}
    (Neuron Balance) For all local minimum of 
    Eq.~\eqref{eq: free energy} and any $i,\ j$, 
    \begin{equation}
\eta\mathbb{E}\Tr[g_{i,j,:}g_{i,j,:}^\top-g_{i+1,:, j}g_{i+1,:,j}^\top]=4\gamma\Tr[w_{i,j,:}w_{i,j,:}^T-w_{i+1,:,j}w_{i+1,:,j}^T].
    \end{equation}
\end{theorem}

From a physics perspective, we have proved an equipartition theorem (ET). The elements in the matrix $\E [g_ig_i^T]$ can be seen as the temperature (or, the average energy) felt by each parameter. The trace $ \Tr \E [g_ig_i^T]$ thus gives the temperature of the layer. That different layers emerge to have the gradient second momentum is an explicit form of the ET and means that the entropy $S$ must be evenly spread out across every layer. Because the standard ET in physics is a property of thermal equilibrium, our result may be seen as an extension of the physical law to the out-of-equilibrium dynamics of learning. See Figure~\ref{fig:balances} for the emergence of layer and neuron balances in a ReLU network. We train on the MNIST dataset, but the labels are generated by a teacher ReLU network and trained with an MSE loss. Also, see Appendix~\ref{app:exp} for examples of training trajectories and for an example with a self-attention net.

\begin{figure}
    \centering
    \includegraphics[width=0.235\linewidth]{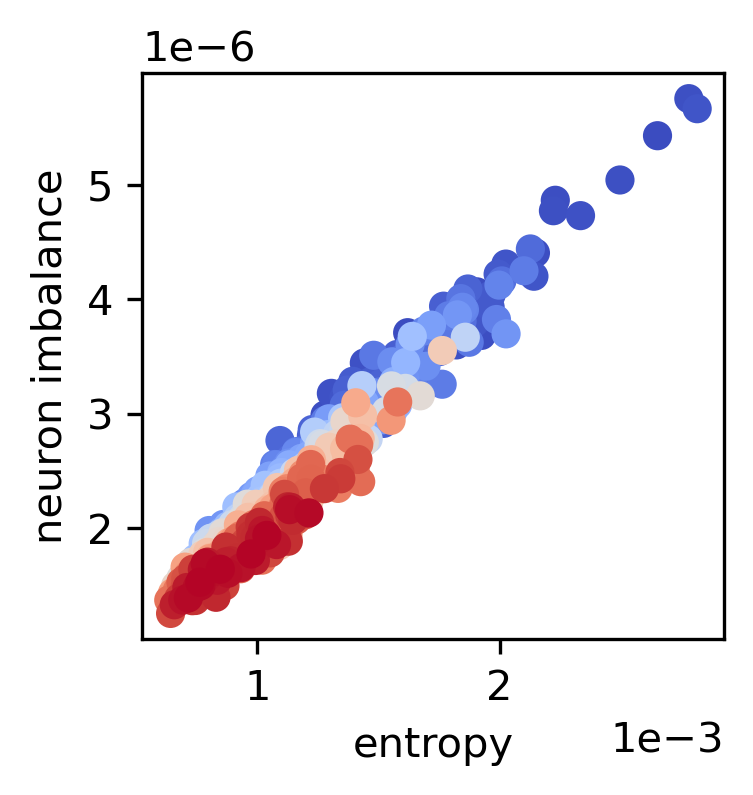}
    \includegraphics[width=0.235\linewidth]{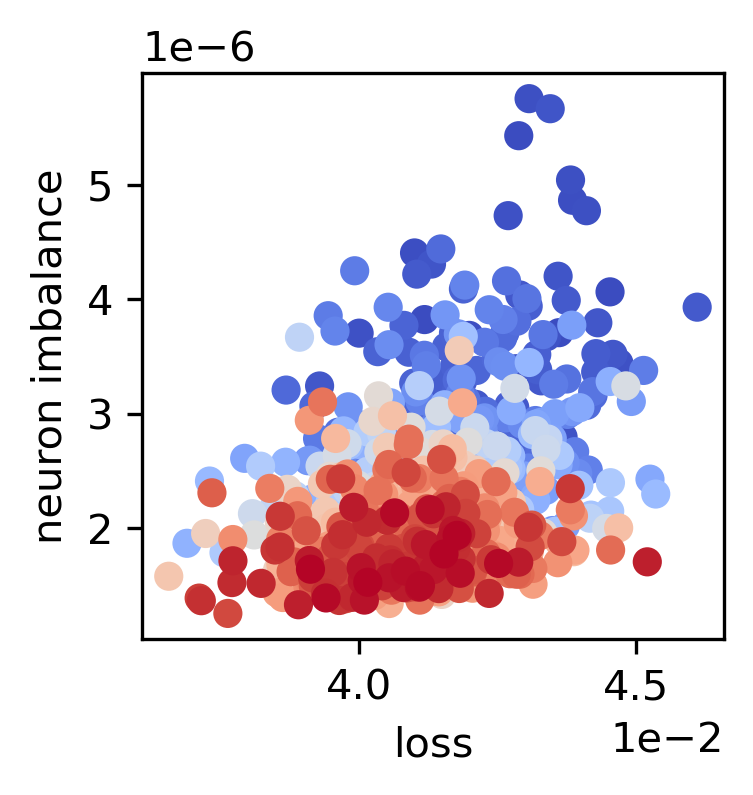}
    \includegraphics[width=0.235\linewidth]{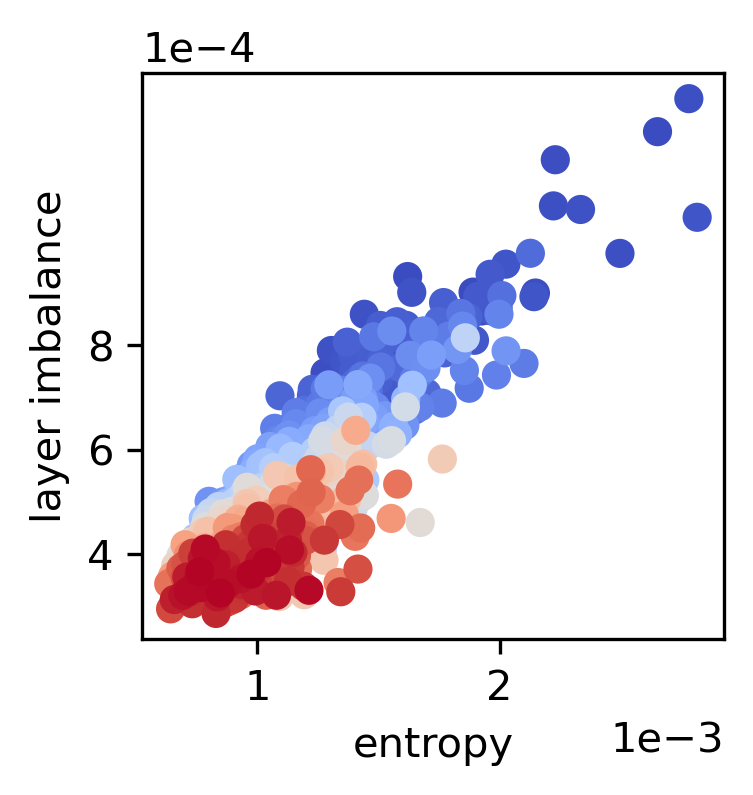}
    \includegraphics[width=0.235\linewidth]{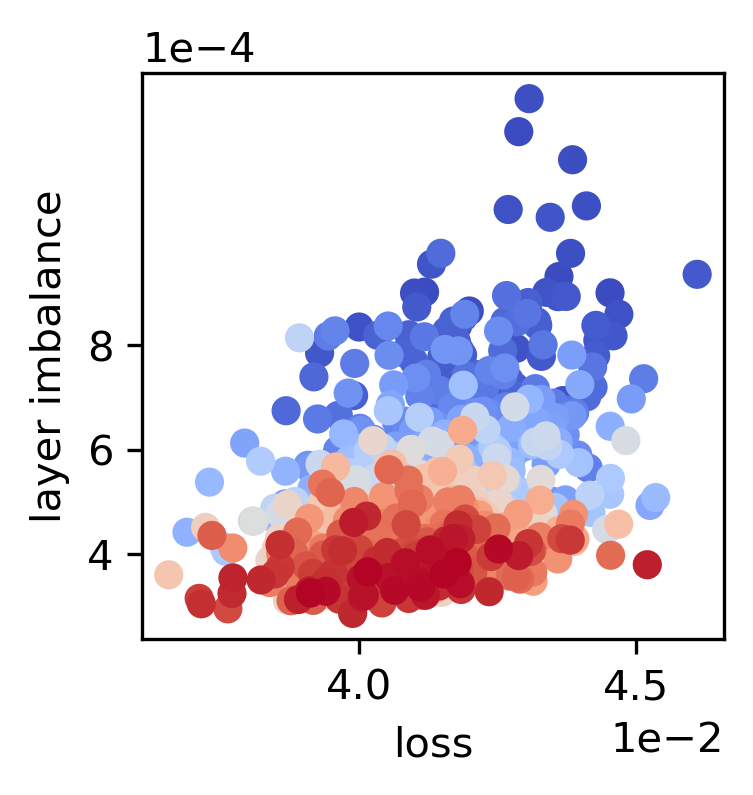}
    \vspace{-2mm}
    \caption{\small Layer and neuron gradient balance during training of a two-layer ReLU network. Here, every dot is a fixed time during training, where bluer dots are closer to the initialization, and redder dots are closer to convergence. \textbf{Left 1-2}: The entropy is strongly correlated with the neuron balance. As entropy decreases, the neuron balance improves. In contrast, the loss is not correlated to entropic effects at all. \textbf{Right}: Similarly, the layer balance is also correlated with entropy and not with loss function value.}
    \label{fig:balances}
\end{figure}

\vspace{-1mm}
\paragraph{Polynomial Network}

Now, consider the case where $R(h)$ is a diagonal matrix such that $R_{ii}(h)= h^d$ corresponds to a polynomial activation. This type of network is also a variant of homogeneous networks \cite{du2018algorithmic}. For these networks, one can show that the converge to a state where every layer's gradient norm is $d$ times that of the previous layer, leading to a gradient exploding or vanishing problem (See Appendix~\ref{app sec: polynomial net}).

\vspace{-1mm}
\paragraph{Self Attention}

Consider the case when a model has a generic form: $\ell(x,W, U) = \ell(x,WU)$, where $W$ and $U$ are matrices. The loss can contain other trainable parameters, which we ignore. Define $G_W = \E_{x\in\mathcal{B}}\nabla_W \ell(x,W,U),\quad G_U =\E_{x\in\mathcal{B}}\nabla_U\ell(x,W,U)$, and one can prove the following relation:
\begin{theorem}
\label{theo:attention}
(Gradient Alignment) For all local minimum of Eq.\eqref{eq: free energy}, we have
\begin{equation}
\eta\E[G_W^T G_W-G_U G_U^T]=4\gamma(W^TW-UU^T).
\end{equation}

\end{theorem}

This theorem is thus applicable to matrix factorization, deep linear networks, and, more importantly, self-attention layers. The self-attention logit is computed as $a_{ij} = X_i^T W U X_j$, where $W$ is the key matrix, $U$ is the query matrix. The loss function is a function of $a_{ij}$ viewed as a matrix: $\ell(\{a_{ij}\})$. Let $V = W_2W_1$. Then, applying this theorem reveals an intriguing relation:
\begin{equation}
    W_{1} \E[ G_V^T G_V] W_1^T = W_{2}^T \E[ G_V G_V^T] W_2.
\end{equation}

\paragraph{Interpolating Weight Balance and Gradient Balance}

For all the theorems above, we have also studied how weight decay affects the balance conditions. We see that the weight decay creates something analogous: instead of gradient balance, weight decay encourages weight balance, and this effect often cannot be achieved together with gradient balance. Thus, there is a tradeoff between gradient balance and weight balance. In reality, the network is somewhere in between, where the weight balance and gradient balance has to ``balance" with each other. Also, this is a generalized form of an equipartition theorem. If we regard the sum of the regularization and $S$ as a ``total" entropic potential $\Gamma$, then this means that every layer will contribute an equal amount to $\Gamma$.\footnote{Now, it is helpful to clarify the role of the $L_2$ regularization term. Conceptually, it can either be regarded as a part of the loss function $\ell$, which we have done so far, or as a part of the entropic term. Treating it as a part of entropy is sometimes conceptually preferred because, like the entropy, it also functions as a regularization of the parameter space, limiting the accessible states.}


\begin{figure}[t!]
    \centering
    \includegraphics[width=0.4\linewidth]{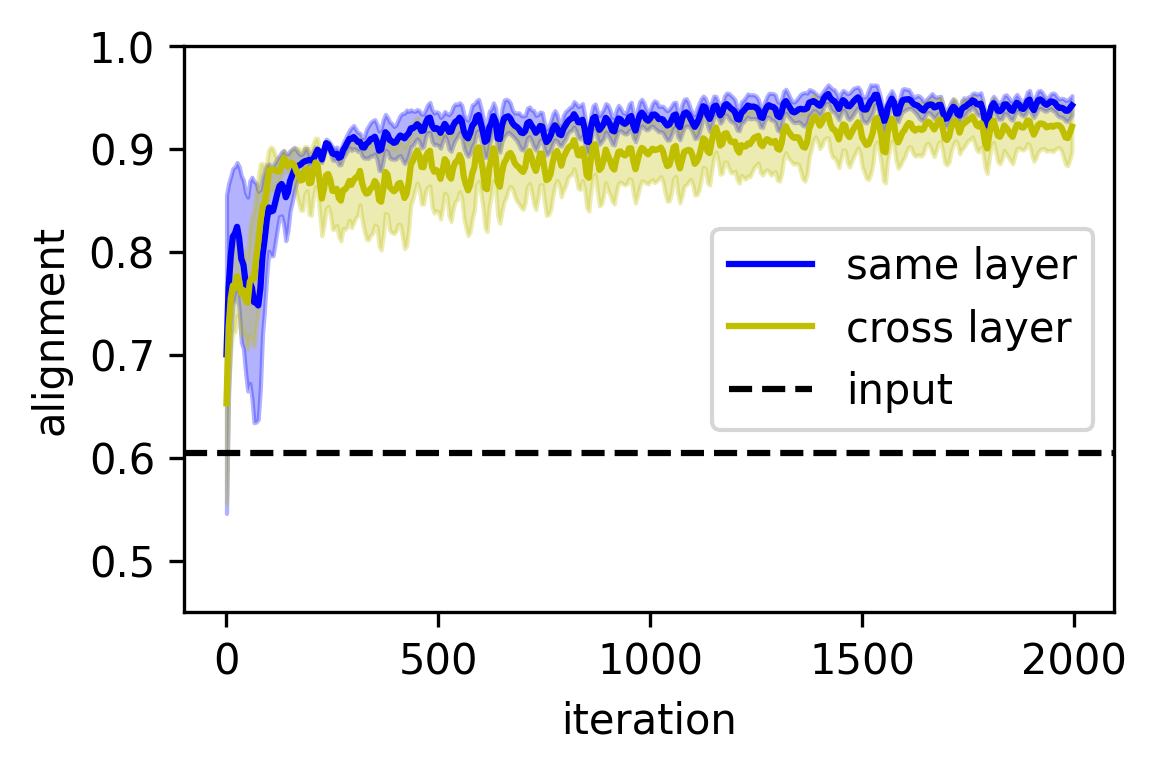}
    \includegraphics[width=0.4\linewidth]{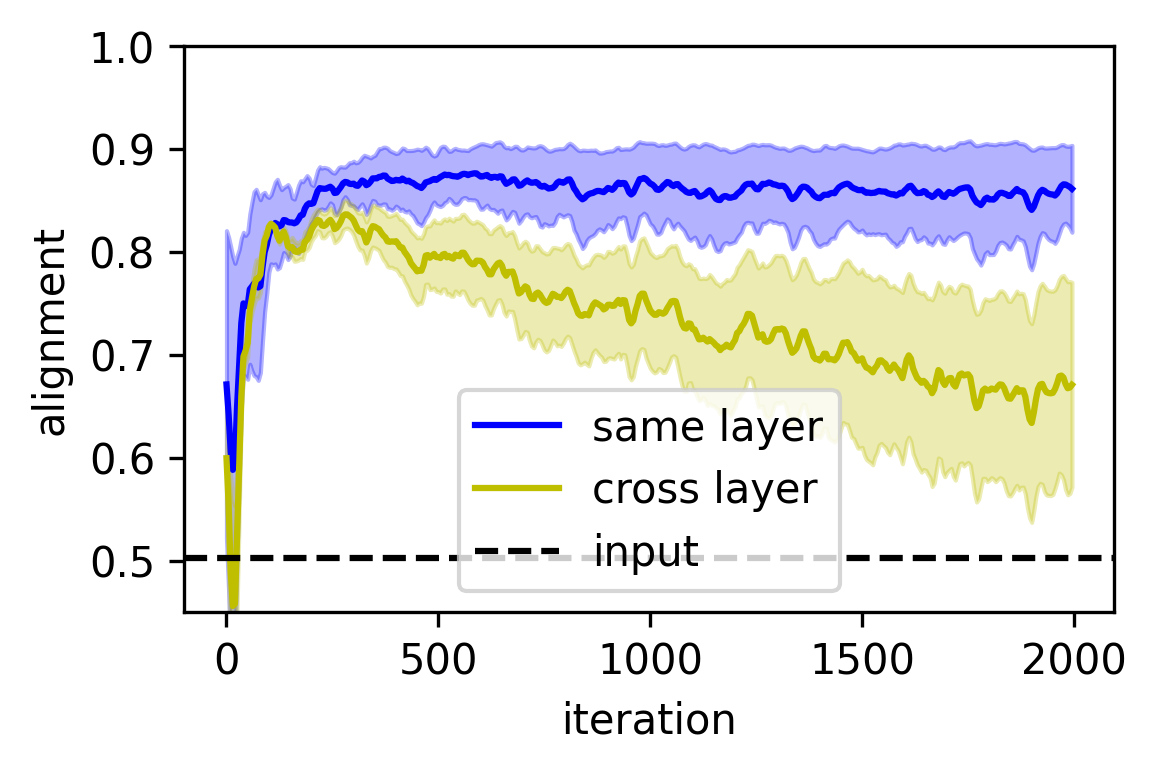}
    \vspace{-2mm}
    \caption{\small The representations of two $6$-layer networks independently trained on randomly transformed MNIST become perfectly aligned for \textbf{every} pair of layers. The figure shows the average alignment between the same or different layers of two networks. This alignment does not weaken even if the input is arbitrarily transformed (Theorem~\ref{theo: universal representation}). The black dashed line shows the average alignment to the input data, which is significantly weaker. \textbf{Left}: linear network. \textbf{Right}: tanh network.}
    \label{fig:representation alignment}

    \centering
    \vspace{1mm}
    \includegraphics[width=0.8\linewidth]{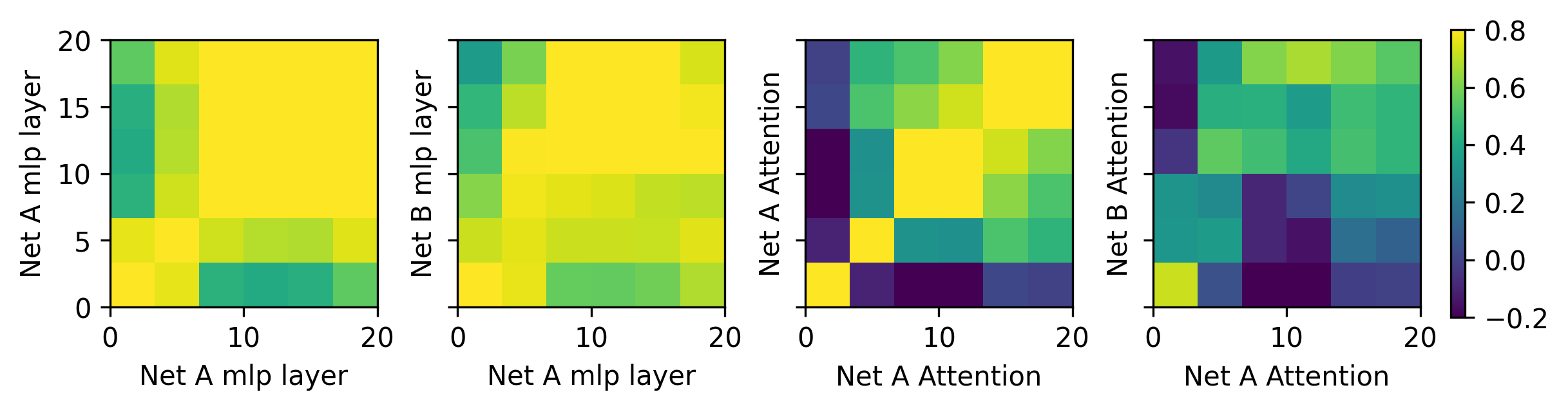}
    \vspace{-2mm}
    \caption{\small Alignment of representations of two ViT models pretrained on ImageNet. Net A: ViT-B (\#param: 86M). Net-B: ViT-H (300M). We see that both mlp layers and the self-attention layers have mutually aligned representations both with itself and the other with each other. In particular, the alignment with itself is slightly better than with the different model, and the alignment of later layers is better than that of the first layers. A similar result with larger models is shown in Appendix~\ref{app sec: vit}.}
    \label{fig:vit alignment}
\end{figure}

\vspace{-1mm}
\section{Implications}\label{sec: universal}
\vspace{-1mm}
Next, we apply these results to study the emergence of universal representations in neural networks, and the progressive sharpening phenomena in deep learning optimization.
\vspace{-1mm}
\subsection{Universal Representation Learning}
\vspace{-1mm}
Recent works found that the representations of learned models are almost universally aligned to different models trained on similar datasets \cite{bansal2021revisiting, kornblith2019similarity}, and even to the biological brains \cite{yamins2014performance}. This interesting phenomenon has a rather philosophical undertone and has been termed ``Platonic Representation Hypothesis" \cite{huh2024platonic}. Here, we say that the two neural networks have learned a universal representation if for all $x_1,\ x_2$, 
\begin{equation}\label{eq: universal alignment}
    h_A(x_1)^T h_A(x_2) = h_B(x_1)^T h_B(x_2),
\end{equation}
where $h_A$ is the activation of network $A$ in one of the hidden layers, and  $h_B$ for network B. This is an idealization of what people have observed -- and the difference between the two sides is the ``degree of alignment." We leverage the entropic force formalism to identify an exact solution to the embedded deep linear network (EDLN) model: 
\begin{equation}
\ell(\theta, x)=||M_1W_D\cdots W_1M_2 x-y(x)||^2.
\label{eq:loss_deep_linear}
\end{equation}
on datasets $\mathcal{D}_{M_3} = \{(M_3x_i, y_i)\}_i$, where $M_1,M_2,M_3$ are fixed but arbitrary invertible matrices. They have the following meaning:
\begin{itemize}[noitemsep,topsep=0pt, parsep=0pt,partopsep=0pt, leftmargin=15pt]
    \item $M_1$ can be seen as a model of the layers coming after the deep linear network;
    \item $M_2$ models the layers coming after the embedded network;
    \item $M_3$ models different views of the data, which is common in multimodal learning -- therefore, the two models are trained on two different (but related) datasets.
\end{itemize}
In Theorem~\ref{theo: universal representation}, we will train two different models, each with their own and potentially different $M_1$, $M_2$, $M_3$. The arbitrariness of these three matrices implies \textit{universality}.

The data is generated by $y_i=Vx_i+\epsilon_i$ for i.i.d. noise $\epsilon_i$. Assuming that $\E x_i=\E\epsilon_i=0$ and $\Sigma_\epsilon:=\E\epsilon_i\epsilon_i^T$, $\Sigma_x:=\E x_ix_i^T$, 
the following theorem characterizes the global minimum of the entropic loss for this network in the case $\gamma=0$ and $\eta \to 0_+$.


\begin{theorem}[Perfect Platonic Representation Hypothesis]\label{theo: universal representation}
Consider two deep linear networks A and B with weights of arbitrary dimensions larger than $\text{rank}(\sqrt{\Sigma_\epsilon}V\sqrt{\Sigma_x})$. Let model A train on $\mathcal{D}_{M_3}$ and model B on $\mathcal{D}_{M_3'}$. Moreover, the outputs of models are multiplied by $M_1,M_1'$, and the inputs are multiplied by $M_2,M_2'$, respectively. Then, at the global minimum of Eq. \eqref{eq: free energy}, every hidden layer of A is \textbf{perfectly} aligned with every hidden layer of B for any $x$, in the sense that
\begin{equation}
h_A^{L_A}(x)=c_0Rh_B^{L_B}(x)
\end{equation}
for $1\leq L_A<D_A$ and $1\leq L_B<D_B$ and any $x$, where
$c_0$ is a scalar and $R=U_1U_2^T$ satisfying $U_1^TU_1=U_2^TU_2=I$. $h_A^{L_A}(x):=\Pi_{i=1}^{L_A}W_i^AM_2M_3x$, $h_A^{L_A}(x):=\Pi_{i=1}^{L_A}W_i^AM_2'M_3'x$ denote the output of the $L_A,L_B-$the layer of network A and B, respectively.

\end{theorem}
Here, a perfect alignment means that $h_A^{L_A}(x)$ differs from $h_B^{L_B}(x)$ only by a scaling and a rotation. Because of symmetry, SGD converges to a state where all possible pairs of the intermediate layers of two different networks are mutually aligned, independent of the initialization. This is an extraordinary fact because there exist infinitely many solutions that are not perfectly aligned, also due to symmetry. For example, take $h_B$ to be any hidden layer, and transform its incoming weight by $A$ and its outgoing weight by $A^{-1}$. This remains a global minimum for $L$, but it is no longer the case that there exists an orthogonal transformation $R$ such that $h_A = h_B$. Therefore, for almost all global minima of $L$, there is no universal alignment between layers -- yet, SGD prefers a universal solution due to the entropy term. The following theorem shows that weight decay will lead to a nonuniversal representation.

From a physics and thermodynamics perspective, it is quite reasonable that the irreversibility of the dynamics leads to the emergence of universal structures. A state is not universal if it contains information about its initial condition. Therefore, the irreversibility of the learning dynamics helps erase information about the parameter initialization, thereby enabling the learning of a universal solution. The following hypothesis can summarize this perspective:
\begin{quote}
    \centering \textit{Irreversibility enables universal representation learning.}
\end{quote}
Now, it is worthwhile to remark on the connection and difference between this result and the original PRH \cite{huh2024platonic}. First, our theory provides strong support for the PRH. The original PRH only hypothesizes a positive similarity between models, and it is unclear whether the alignment score can reach 1 (which implies a perfect alignment) or will only be a small positive value. Our result shows that, in principle, it is possible to reach the perfect alignment limit. Secondly, this result also offers an alternative perspective on the representation alignment phenomenon to that of the original PRH paper. In the old perspective, one regards having no alignment as the default expectation, and positive alignment as something to be explained and understood. In our new perspective, the perfect alignment is the default, and breaking away from it is something to be explained and understood -- which is exactly what we will demonstrate in the next theorem. We also refer to Ref.~\cite{ziyin2025proof} for a more instructive derivation of the proof and for a detailed discussion of different ways to break the perfect PRH.

\begin{theorem}
\label{theo:deep_linear_wd}
Consider the deep linear network \eqref{eq:loss_deep_linear} with widths larger than $d:=\text{rank}(V)$. Let $\eta=0$ and $\gamma\to0_+$. At the global minimum of \eqref{eq: free energy}, we have
\begin{equation}
W_i=U_iP_i\Sigma U_{i-1}^T
\end{equation}
for $i=1,\cdots,D$. $U_D$ and $U_0$ are given by the SVD of $M_1^{-1}VM_3^{-1}M_2^{-1}:=U_DSU_0$, where $S\in\mathbb{R}^{d\times d}$ contains the singular values. For $i=2,\cdots,D-1$, $U_i$ are arbitrary matrices satisfying $U_i^TU_i=I_{d\times d}$. Moreover, $\Sigma=S^{1/D}$. $\{P_i\}_{i=1}^D\subset\mathbb{R}^{d\times d}$ are diagonal matrices containing $\pm1$ and satisfy $\Pi_{i=1}^DP_i=I_d$.
\end{theorem}
The universal representation property (Theorem \ref{theo: universal representation}) does not hold anymore. Therefore, gradient balances lead to universal representations, whereas weight balances do not. See Figure~\ref{fig:representation alignment} for an experiment with deep linear and nonlinear networks. A surprising aspect is that every layer can be aligned with every other layer. Because any deep nonlinear network is approximated by a linear network for a small weight norm \cite{ziyin2022exact}, one could say that any nonlinear network is, to first order in $\|\theta\|$, a universal representation learner. In Figure~\ref{fig:vit alignment}, we compare the alignment between different self-attention layers of two differently sized vision transformers pretrained on ImageNet. Note that different layers of the same network and different networks also have significantly positive alignments, consistent with the solution for the embedded deep linear net. That using weight decay does not lead to universal representations is supported by the result in Figure~\ref{fig:representation alignment2}.

Theorem~\ref{theo: universal representation} is a direct (perhaps the first) proof of the Platonic representation hypothesis, implying that for any $x_1, x_2$, Eq.~\eqref{eq: universal alignment} holds. Importantly, the mechanism does not belong to any previously conjectured mechanisms (capacity, simplicity, multitasking \cite{huh2024platonic}). This example has nothing to do with multitasking. The result holds for any deep linear network, all having the same capacity and the same level of simplicity, because all solutions parametrize the same input-output map. Here, the cause of the universal representation is symmetry alone: in the degenerate manifold of solutions, the training algorithm prefers a particular and universal one. This example showcases how symmetry is indeed an overlooked fundamental mechanism in deep learning.

\begin{figure}
    \centering
    \includegraphics[width=0.33\linewidth]{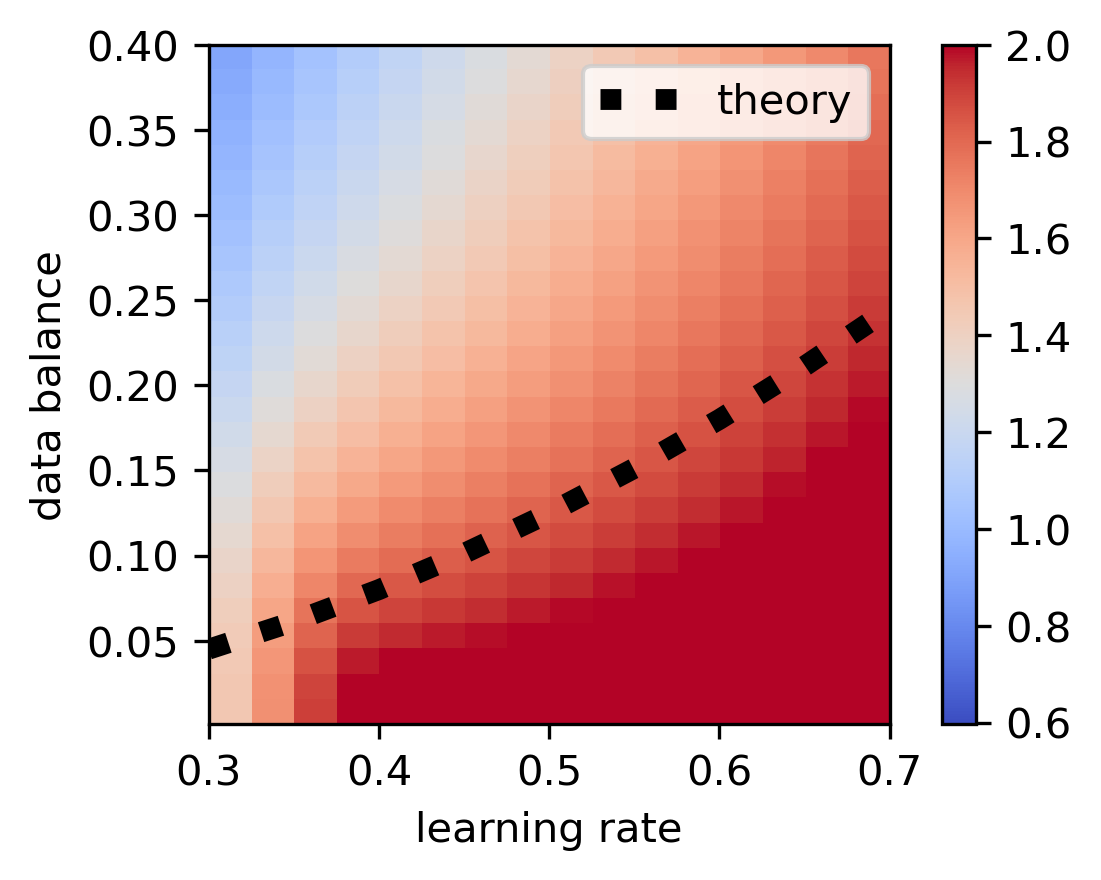}
    \includegraphics[width=0.29\linewidth]{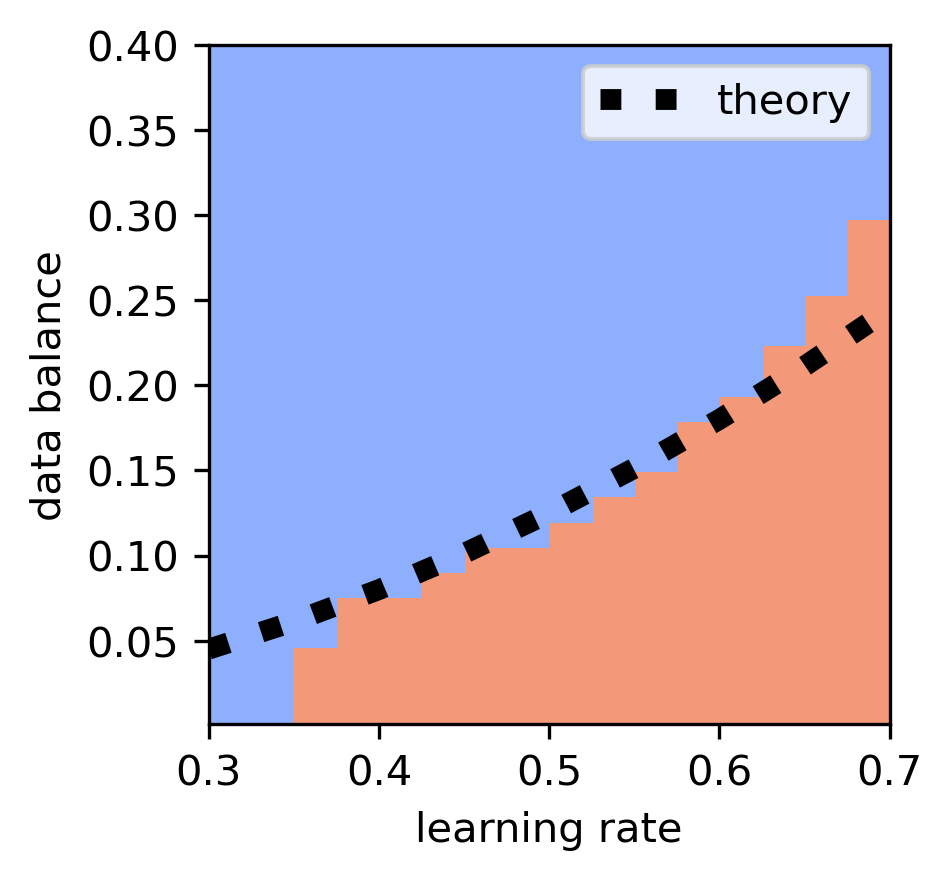}
    \includegraphics[width=0.35\linewidth]{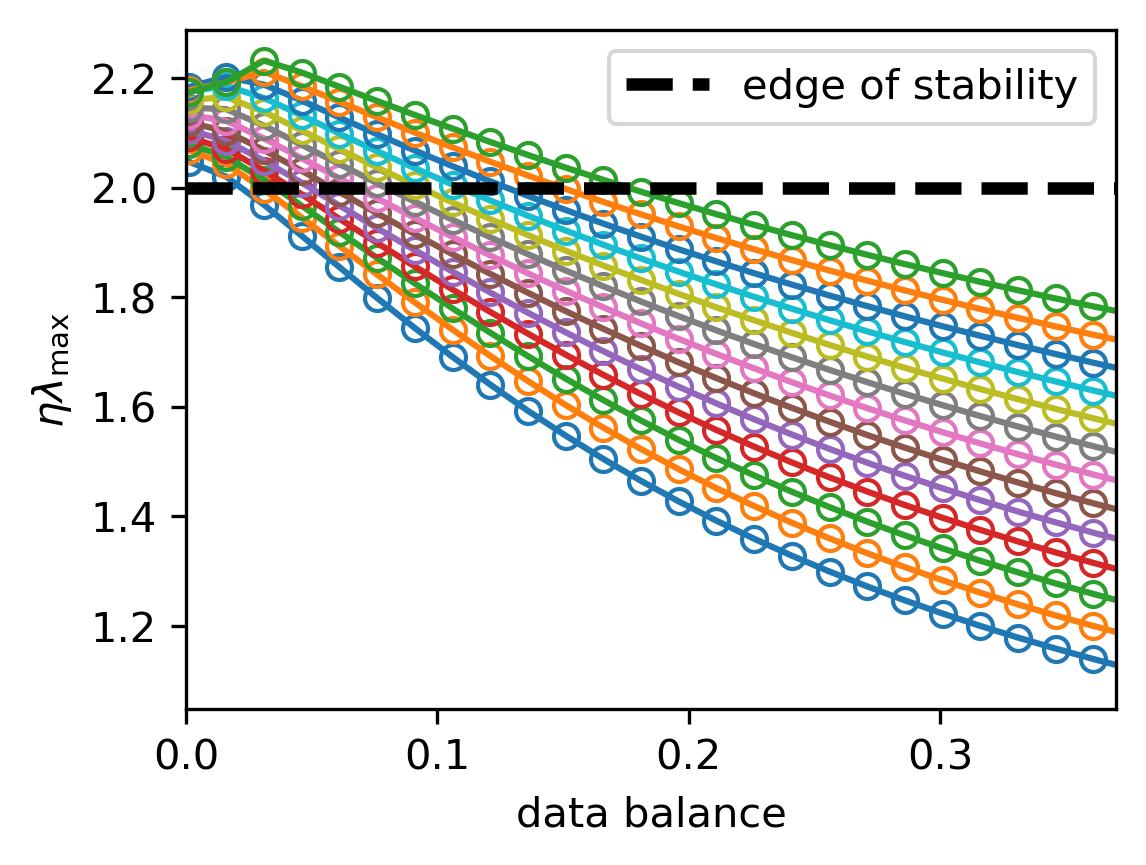}
    \vspace{-2mm}
    \caption{\small The entropic theory predicts the boundary for the edge of stability (EOS) phenomenon \cite{cohen2021gradient}. The theory shows that the imbalance of features and the uncertainty of labels make the model converge to sharper solutions. We run a two-layer linear network trained on a regression task. The \textbf{Left} panel plots the quantity $\eta \lambda_{\rm max}$ at convergence. For stability, $\eta \lambda_{\rm max}$ must stay (approximately) below $2$, and the black dotted line plots the theoretical boundary for $\eta \lambda_{\rm max}=2$. \textbf{Middle}: The same figure that emphasizes the phase boundary. The blue-red boundary empirically defined by the condition $\eta \lambda_{\rm max} = 2 -\epsilon$ with $\epsilon=0.1$ -- due to random sampling, the actual edge of stability is slightly smaller than 2 \cite{liu2021noise}. \textbf{Right}: We control the learning rate and balance of the label noise for this experiment. As predicted, as the data noise becomes more balanced, the sharpness metric $\eta \lambda_{\rm max}$ gets smaller, indicating better dynamical stability during training.}
    \label{fig:eos phase diagram}

    \vspace{1em}
    \centering
    \includegraphics[width=0.24\linewidth]{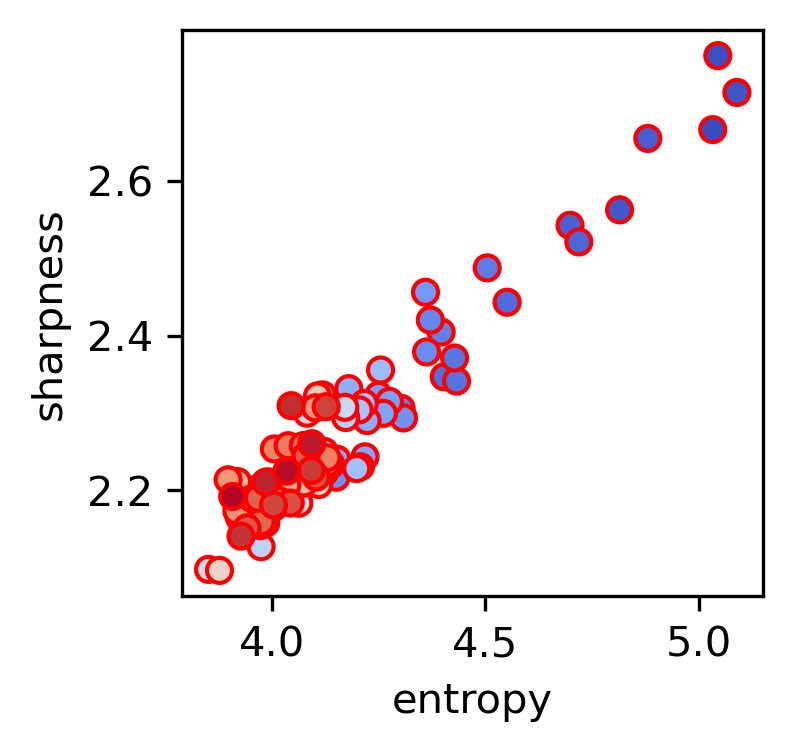}
    \includegraphics[width=0.24\linewidth]{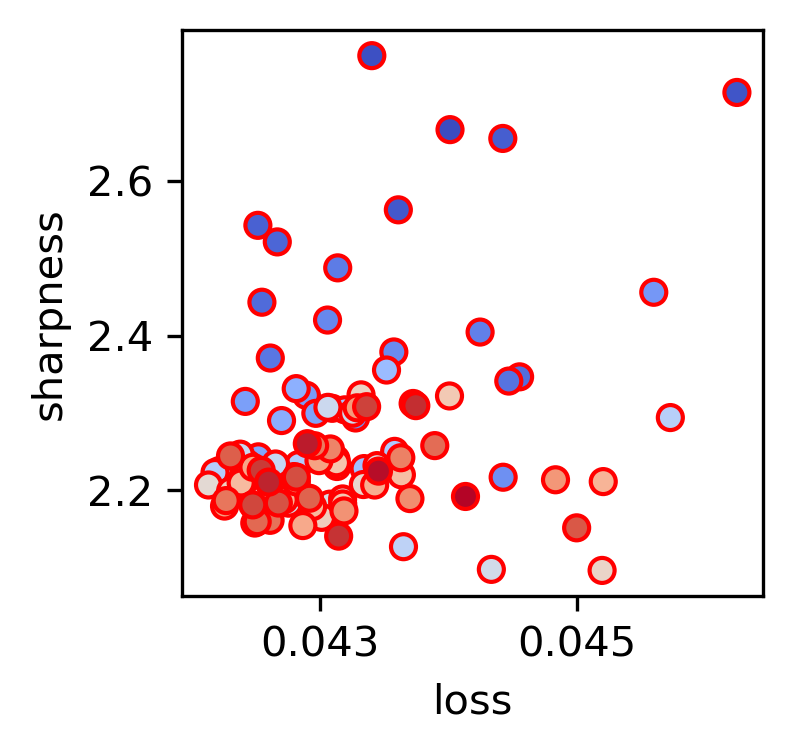}
    \includegraphics[width=0.24\linewidth]{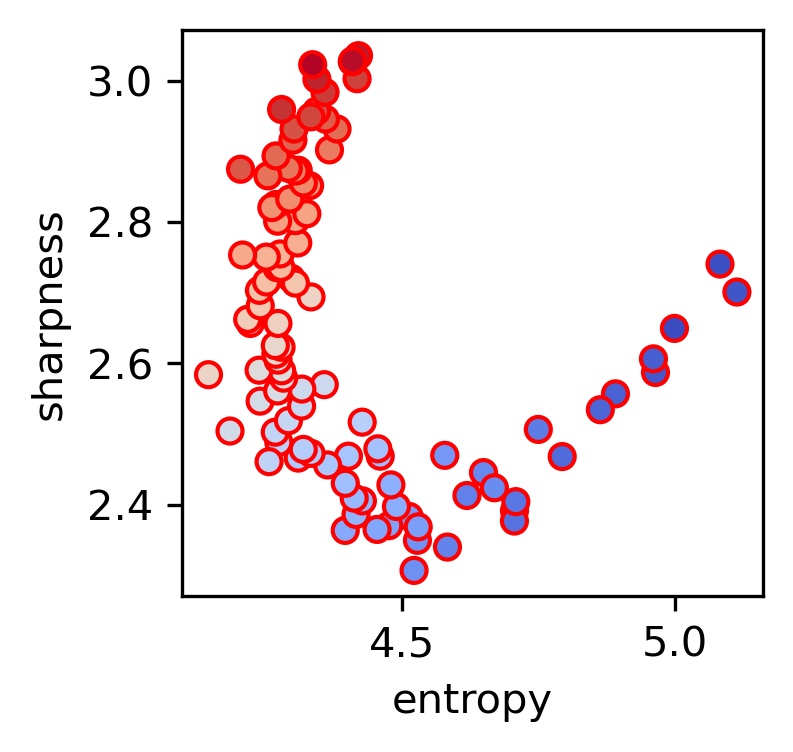}
    \includegraphics[width=0.24\linewidth]{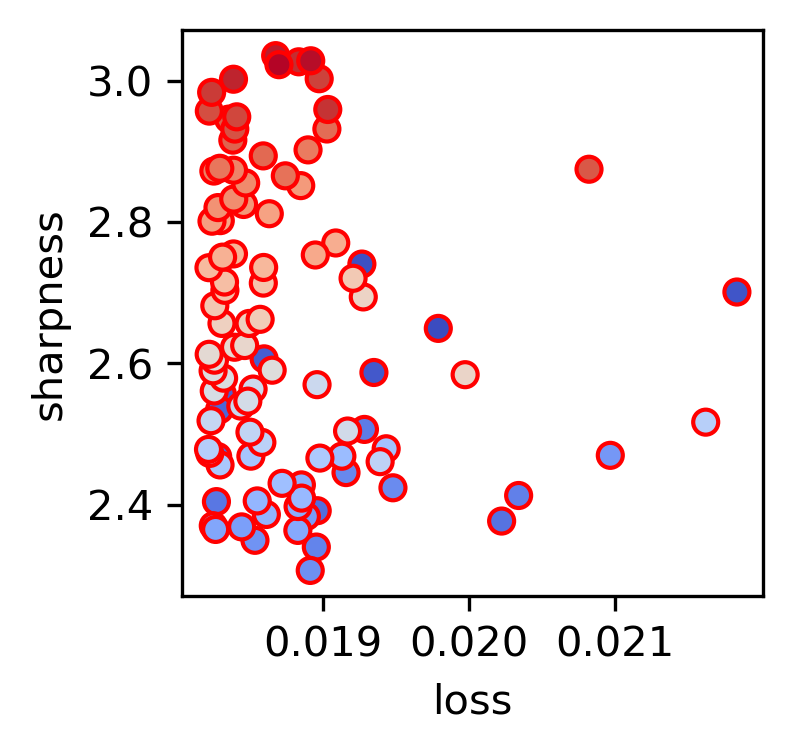}    
    \vspace{-2mm}
    \caption{\small Example of a trajectory of training during the progressive flattening (\textbf{Left 1, 2}) and sharpening (\textbf{Right 1, 2}) of the experiment in Figure~\ref{fig:eos phase diagram}. Here, blue dots correspond to earlier in the training, and red dots correspond to later in the training. During progressive flattening, the decrease in sharpness correlates directly with the entropy term (Left 1), whereas the training loss is independent of the sharpness (Left 2). For progressive sharpening (Right), the picture is more complicated. The training trajectory follows three phases. Phase I: flattening correlates with a decrease in entropy; phase II: sharpening correlates with a decreasing entropy; phase III: sharpening correlates with an increasing entropy. The phase III cannot be explained by the leading-order entropic loss because, as the landscape becomes sharper, higher-order effects in $\eta$ start to dominate training. At the same time, the loss is never correlated with these effects (Right 4).}
    \label{fig:eos dynamics}
\end{figure}

\vspace{-1mm}
\subsection{The Sharpness Paradox}
\vspace{-1mm}

The tendency towards learning universal representations can imply a curse for training. For example, the sharpness of the loss landscape really depends on the distribution of the data, whereas the solution the model finds is independent of these distributions -- this could imply that these solutions can be quite bad in terms of, say, optimization properties. Meanwhile, a paradox of the sharpness-seeking behavior of SGD has become explicit. On the one hand, the edge of stability (EOS) states that learning \textit{typically} leads to sharper solutions, whereas a vast majority of works have shown that SGD training leads to flatter solutions \cite{damian2021label, wu2022alignment}. These two cannot happen simultaneously -- and the solution must be that the sharpness-seeking behavior of SGD is situation-dependent. This section formalizes this intuition and shows that universal representation learning can be intrinsically related to the edge of stability phenomenon, which is also ubiquitous in deep learning.

\begin{definition}
The total sharpness is defined as $T(\theta)=\Tr\E\nabla^2\ell(x,\theta)$.
\end{definition}
This definition is chosen for analytical tractability and has been used in prior works \cite{li2021happens, ziyin2024parameter}. $T$ upper bounds the largest eigenvalue, and $T/d$ lower bounds it, so it is a good metric of stability and sufficient for the theorem we will prove. The following lemma shows that if there is an exponential symmetry in $\ell$, every local minimum of $\ell$ connects without barrier to a local minimum with arbitrarily high sharpness. It can be seen as a generalization of the result of Ref.~\cite{Dinh_SharpMinima} to general symmetries.

\begin{lemma}
\label{lemma:sharpness}
    (Sharpness Lemma) Assume that $A$ is a symmetric matrix and $\ell(x,e^{\lambda A}\theta)=\ell(x,\theta)$. Moreover, assume that $A\mathbb{E}\nabla^2\ell(x,\theta)\neq0$. Then, $\limsup_{|\lambda|\to+\infty}|T(e^{\lambda A}\theta)|=+\infty$.
\end{lemma}

One can analytically solve for the sharpness of two-layer linear networks, and identify a precise cause of the progressive sharpening effect. 
\begin{theorem}\label{theo: sharpness}
For a two-layer linear network $\ell(x,\theta)=||y(x)-W_1W_2x||^2$ with $y(x)=Vx+\epsilon\in\mathbb{R}^{d_y}$ and $\E x=\E\epsilon=0$, $\Sigma_x=\E xx^T$, $\Sigma_\epsilon=\E\epsilon\epsilon^T$. Denote $\tilde{U}S'\tilde{V}$ to be the SVD of $V':=\sqrt{\Sigma_\epsilon}V\sqrt{\Sigma_x}$ and assume that the width of the network is larger than $\text{rank}(V')$. Then we have
\begin{equation}
T(\theta)=d_y\sqrt{\frac{\Tr[\Sigma_x]}{\Tr[\Sigma_\epsilon]}}\Tr[S']+\sqrt{\Tr[\Sigma_x]\Tr[\Sigma_\epsilon]}\Tr[\Sigma_\epsilon^{-1}\tilde{U}S'\tilde{U}^T]
\end{equation}
at the global optimum of \eqref{eq: free energy}. Meanwhile, the minimal sharpness of the global minimum of $\ell$ is
\begin{equation}
\min T(\theta)=2\sqrt{d_y\Tr\Sigma_x}\Tr\hat{S},
\label{eq:S}
\end{equation}
where $\hat{S}$ is the singular values of $V\sqrt{\Sigma_x}$.
\end{theorem}


This result implies that SGD has no inherent preference for flatter minima. See Figure~\ref{fig:eos phase diagram}-\ref{fig:eos dynamics}, where we train a two-layer linear network on a linear regression task with a 2d label $y\in\mathbb{R}^2$. The labels $y= V^*x +\epsilon$, for a ground truth matrix $V^*$ and iid zero-mean noise $\epsilon$ such that $\Sigma_\epsilon={\rm diag}(1, \phi_x)$, where $\phi_x \in (0, 1)$ is called the ``data balance." In Appendix~\ref{app sec: eos}, we also train a deep nonlinear network, and we see the same trend where improving balance in the label noise leads to flatter solutions.

As an example, we can choose $\Sigma_x=I$, $V=I$ ($d_x=d_y=d$), which gives $V'=\sqrt{\Sigma_\epsilon}=\tilde{U}S'\tilde{U}^T$, and thus
\begin{equation}
T(\theta)=d^{3/2}\Tr[\Sigma_\epsilon]^{-1/2}\Tr[\Sigma_\epsilon^{1/2}]+d^{1/2}\Tr[\Sigma_\epsilon]^{1/2}\Tr[\Sigma_\epsilon^{-1/2}],
\end{equation}
which can be arbitrarily large. Recall that the minimum of $T(\theta)$, on the other hand, does not depend on $\Sigma_\epsilon$, the label noise covariance. Thus, the imbalance of the noise spectrum can lead to arbitrarily high sharpness. This could especially be a problem for language model training because there is a large variation in the randomness of tokens. Some words, like ``the," could have a very low conditional entropy, while nouns or verbs can have high entropy, especially when there exist synonyms. Another example is to choose $\Sigma_\epsilon=I$, which gives $V'=V\sqrt{\Sigma_x}=\tilde{U}S'\tilde{U}^T$, and thus
\begin{equation}
T(\theta)=2\sqrt{d}\Tr[\Sigma_x]^{1/2}\Tr[V\Sigma_x^{1/2}],
\end{equation}
which is exactly the same as the minimal sharpness. This suggests that without the imbalance of the label noise alone, SGD indeed converges to the flattest solution.

Thus, an imbalance in the input feature can lead to different sharpness-seeking behaviors. At the same time, if the loss function has scale invariance ($\ell(x,\theta)=\ell(x, \lambda\theta)$ for any $\lambda\in\mathbb{R}$), the learning dynamics leads to a flattening of the curvature during training (See Section~\ref{app sec: scale inv}). Thus, entropic forces and symmetry are strong factors deciding the sharpness-seeking behavior of SGD.\footnote{Viewed together with the result in Section~\ref{sec: universal}, one reaches an interesting and surprising conjecture: progressive sharpening and universal representation alignment, with entirely different phenomenology,  may have the same underlying cause, and could be two sides of the same coin.}

\vspace{-1mm}
\section{Conclusion}
\vspace{-1mm}

In this work, we have proposed an entropic-force perspective on neural network learning. We derive an entropic loss, which breaks continuous symmetries and preserves discrete ones, leading to universal behaviors such as gradient balance and alignment. The entropic loss suggests a potential unifying perspective for understanding training: learning algorithms prefer solutions with the minimal gradient fluctuation. This perspective provides a unifying framework that explains several emergent phenomena in deep learning, including sharpness-seeking behavior and universal feature structure, as consequences of underlying symmetries and entropic forces. The framework offers predictive and explanatory power across architectures and scales, and points toward a more principled, physics-inspired understanding of learning dynamics and emergent phenomena. Future work may extend this foundation to encompass higher-order corrections, richer architecture structures, and nonequilibrium dynamics of modern training procedures. A major limitation of our work is that we focused only on problems with explicit symmetries; an important future direction is to extend the results to cases with only approximate symmetries. Our experiments show that symmetry-free systems qualitatively agree with symmetry-preserving systems, suggesting that there are underlying, hidden concepts yet to be discovered.

An implication of Theorem~\ref{theo: symmetry breaking} is that discrete symmetries such as $Z_2$ still remain in the loss function. This enables the possibilities of spontaneous symmetry breaking and phase transitions in neural networks. Prior works have studied phase transitions \cite{ziyin2024symmetry} without the entropy term and are inherently zero-temperature phase transitions. It may be possible to develop a theory of phase transitions based on entropic loss, similar to the classical Landau theory. Also, a striking aspect of our construction is its similarity to the actual formalism of thermodynamics in physics; our result motivates the development of a robust thermodynamics theory of deep learning.

\section*{Acknowledgment}
The authors thank Prof. Phillip Isola and Prof. Tomaso Poggio for discussion. ILC acknowledges support in part from the Institute for Artificial Intelligence and Fundamental Interactions (IAIFI) through NSF Grant No. PHY-2019786. This work was also supported by the Center for Brains, Minds and Machines (CBMM), funded by NSF STC award  CCF - 1231216.


\bibliographystyle{unsrt}
\bibliography{neurips_2025}

\clearpage
\appendix

\section{Experiment}
\label{app:exp}
\subsection{ResNet}
For Figure \ref{fig:resnet}, we train ResNet18 on CIFAR 10 using SGD with momentum $0.9$, batchsize $128$ and weight decay $5\times10^{-4}$. The learning rate is $0.1$ at the beginning, $0.01$ after the $100-$th epoch and $0.001$ after the $150-$th epoch. The entropy is calculated by summing the gradient norm of all parameters. We obtain training accuracy $98\%$ and test accuracy $88\%$ at the end.

\subsection{Gradient Balance}
For Figure \ref{fig:balances}, we train on the MNIST dataset but the labels are generated by a teacher ReLU network and trained with an MSE loss. Namely, the loss is
\begin{equation}
    \|f(\theta) -y(x) - \epsilon_x\|^2,
\end{equation}
where $y(x)$ is a parameterized by a random ReLU teacher network and $\epsilon_x$ is an i.i.d. Gaussian noise with $0.2$ standard deviation. The training proceeds with SGD for $10^4$ steps with a learning rate of $0.01$ and batchsize of $200$. Layer balance is calculated by $|\E\Tr [g_ig_i^T-g_{i+1}^Tg_{i+1}]|$ and neuron balance is calculated by $\sum_j|\E\Tr[g_{i,j,:}g_{i,j,:}^\top-g_{i+1,:, j}g_{i+1,:,j}^\top]|$ for the $i-$th layer.

Additional experiments on layer balance and neuron balance are presented in Figures \ref{fig:balance_linear}, \ref{fig:balance_relu} and \ref{fig:balance_attention} for 4-layer linear networks, ReLU networks and simple self-attention networks, still for the teacher-student setting. The hidden dimensions are $256,128,64$ for linear and ReLU networks, and $256$ for self-attention. We present the evolution of layer and neuron balance along training. Figures \ref{fig:balance_linear} and \ref{fig:balance_relu} suggest that layer balance and neuron balance approach zero during training, which verifies Theorems \ref{theo:layer_balance} and \ref{theo: neuron balance}.

\begin{figure}[t!]
    \centering
    \includegraphics[width=0.4\linewidth]{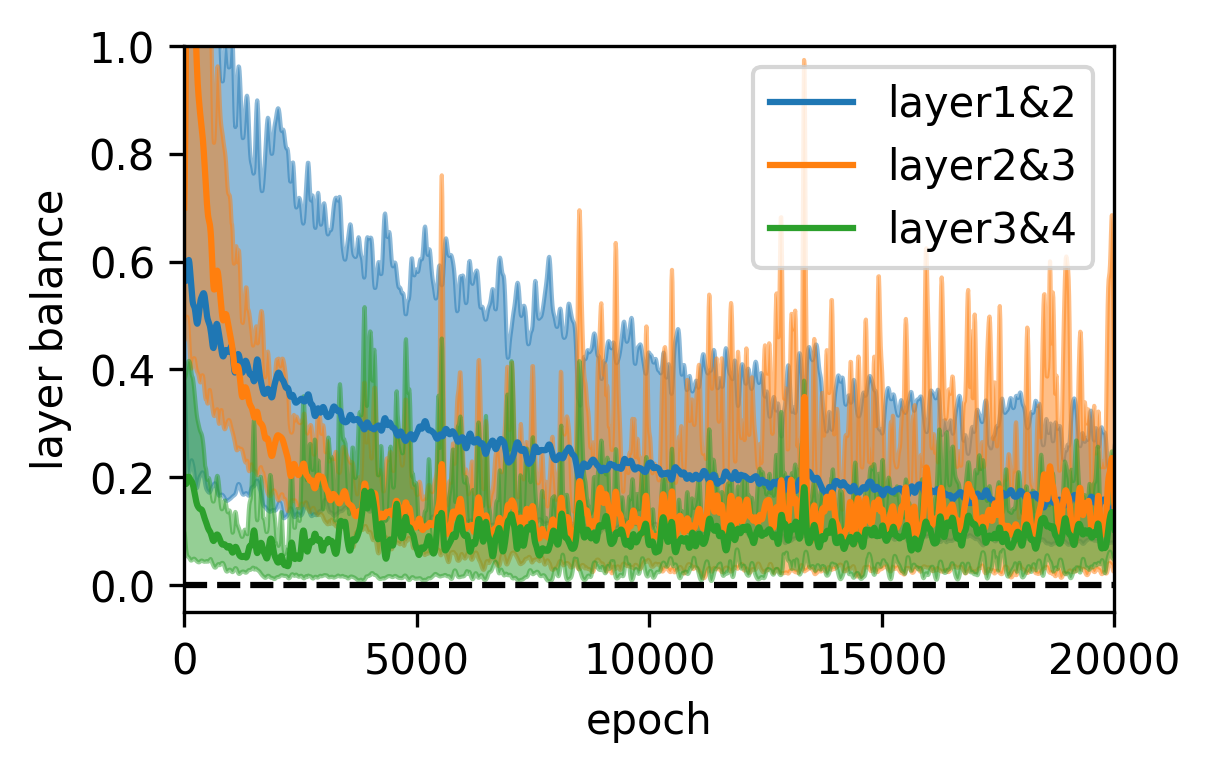}
    \includegraphics[width=0.4\linewidth]{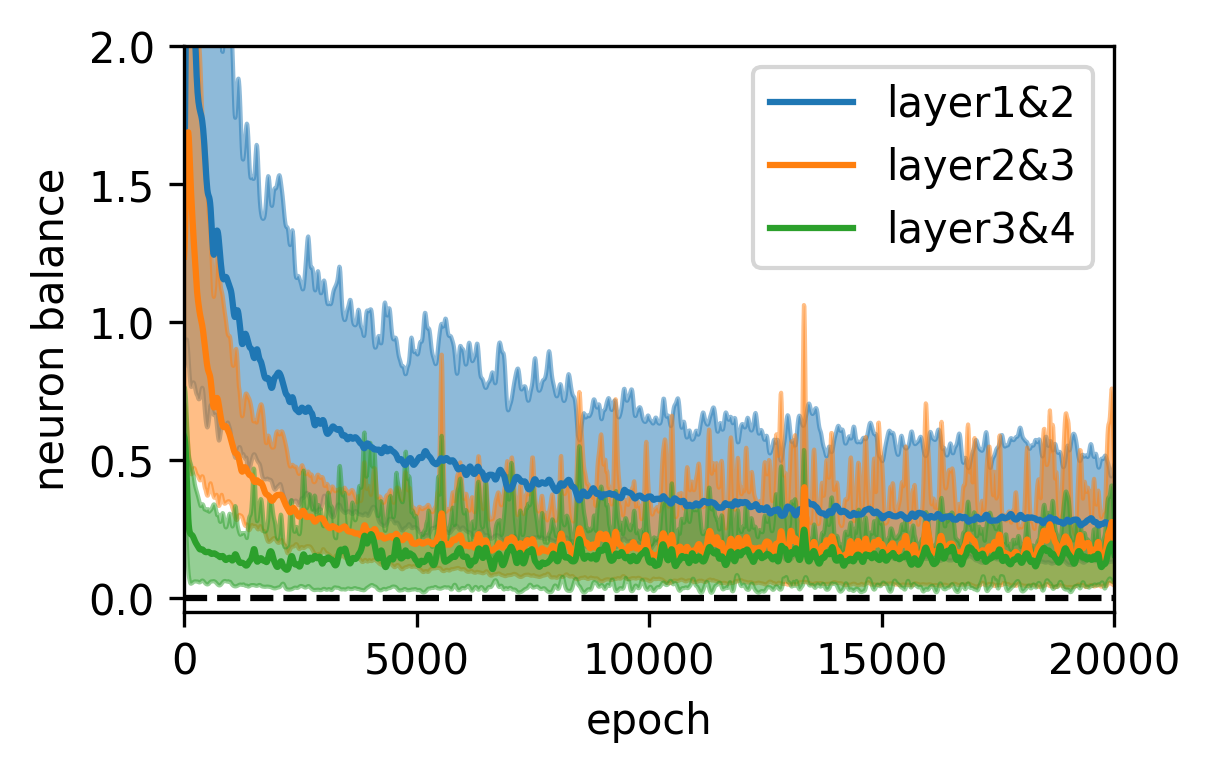}

    \caption{\small 4-layer linear networks and no weight decay trained on a teacher-student setting. \textbf{Left}: Layer imbalance for each layer, which verifies Theorem \ref{theo:layer_balance}. \textbf{Middle}: Neuron balance for each layer, which verifies Theorem \ref{theo: neuron balance}. The curves are smoothed and averaged over $5$ runs for better visualization. \textbf{Right}: Loss and entropy.}
    \label{fig:balance_linear}
\end{figure}

\begin{figure}[t!]
    \centering
    \includegraphics[width=0.4\linewidth]{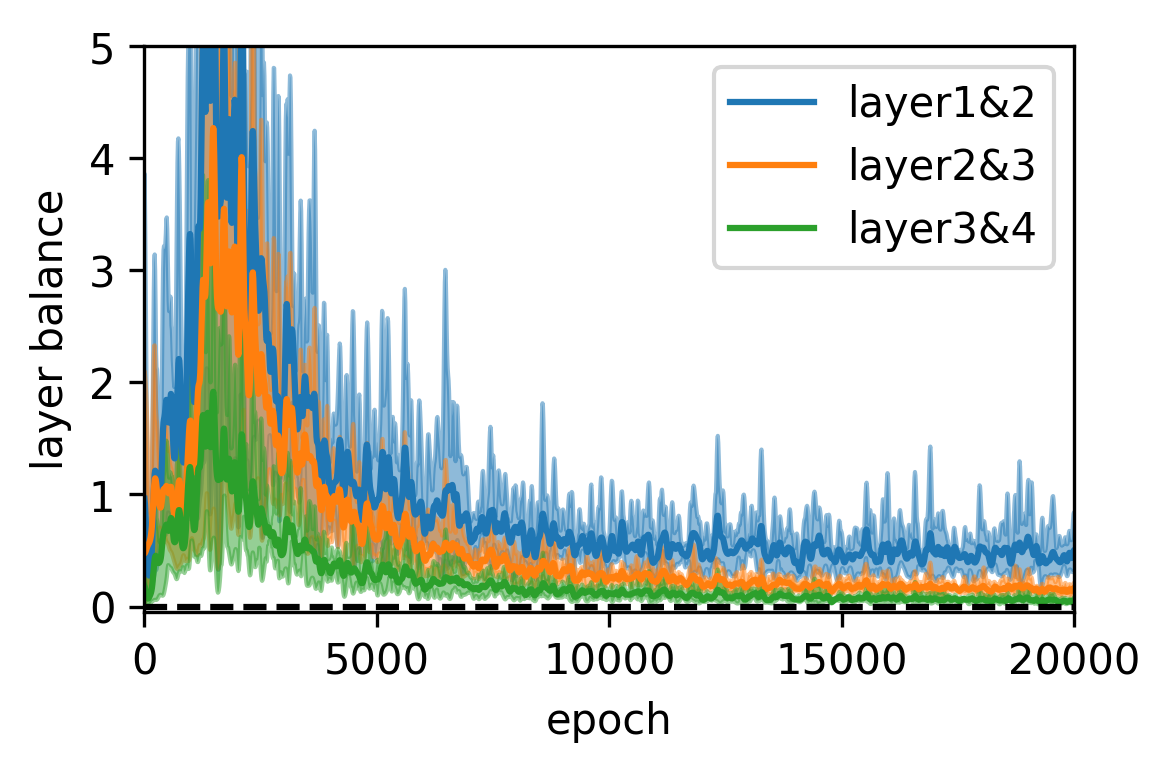}
    \includegraphics[width=0.4\linewidth]{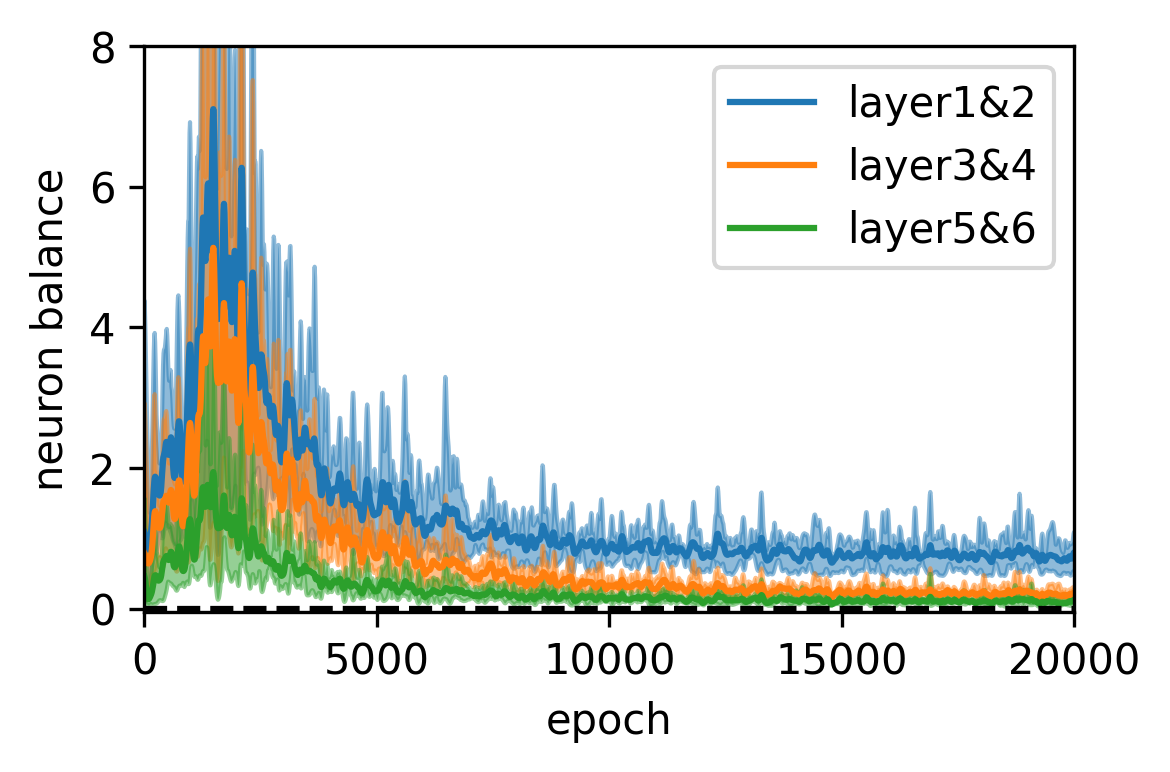}

    \caption{The same setting as Figure \ref{fig:balance_linear}, but for ReLU activation.}
    \label{fig:balance_relu}
\end{figure}

\clearpage

\subsection{Gradient Balance in Self-Attention Nets}

Figure \ref{fig:balance_attention} suggests that $G_U^TG_U-G_VG_V^T$ approaches zero during training, and it is correlated with the entropy rather than the loss, which verifies Theorem \ref{theo:attention}.

\begin{figure}[t!]
    \centering
    \includegraphics[width=0.3\linewidth]{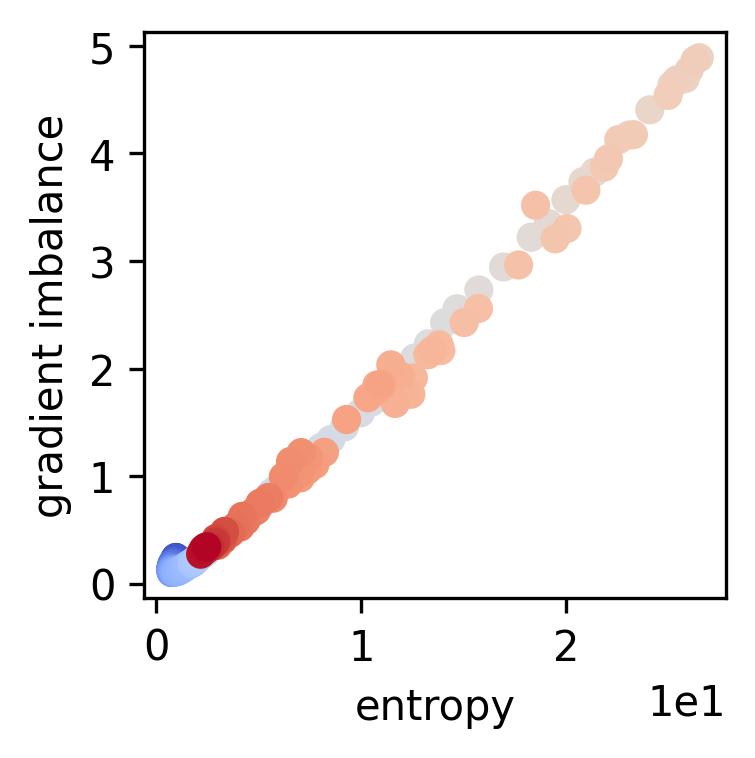}
    \includegraphics[width=0.3\linewidth]{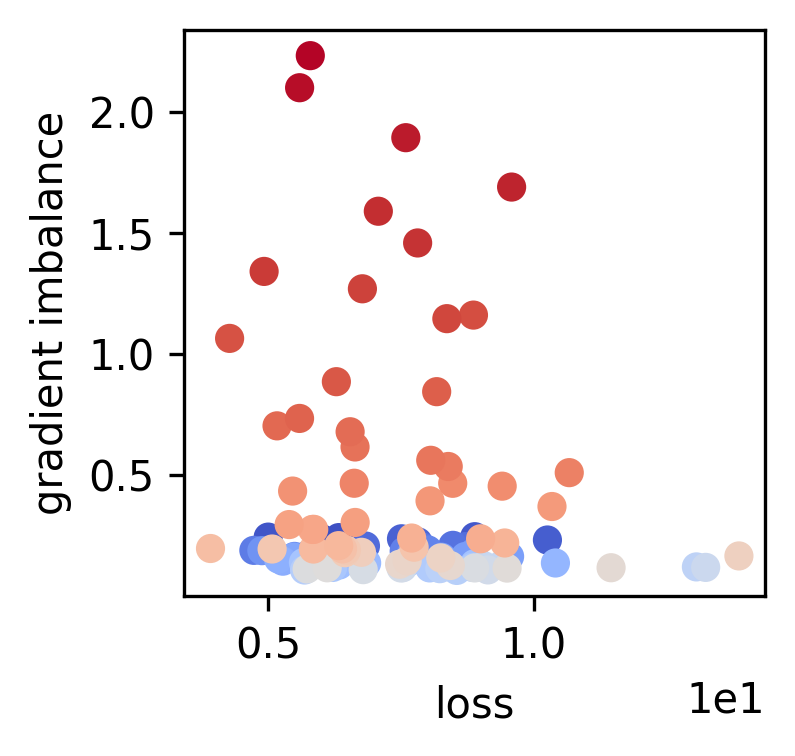}
    \includegraphics[width=0.35\linewidth]{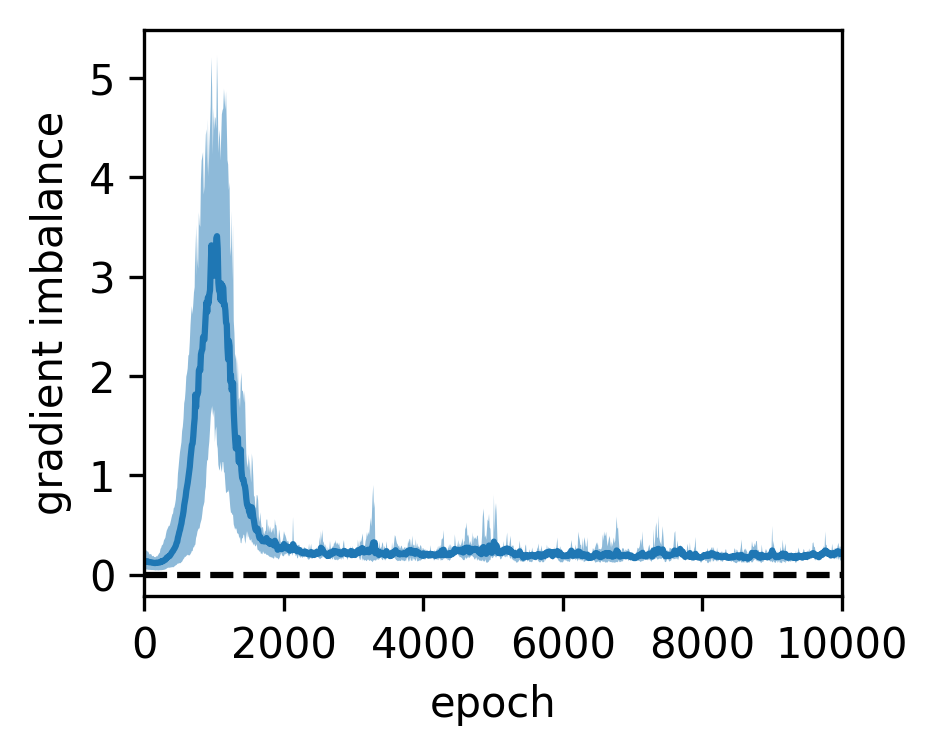}
    
    \caption{\small A self attention networks $y=(x^TUVx)w^Tx$ and no weight decay trained on a teacher-student setting. The gradient imbalance is evaluated by $||G_U^TG_U-G_VG_V^T||_F$. \textbf{Left}: Gradient imbalance is strongly correlated with the entropy. \textbf{Middle}: Gradient imbalance is weakly correlated with the entropy. \textbf{Right}: The evolution of entropy along training, which is averaged over $5$ runs.}
    \label{fig:balance_attention}
\end{figure}

\clearpage

\begin{figure}[t!]
    \centering
    \includegraphics[width=0.4\linewidth]{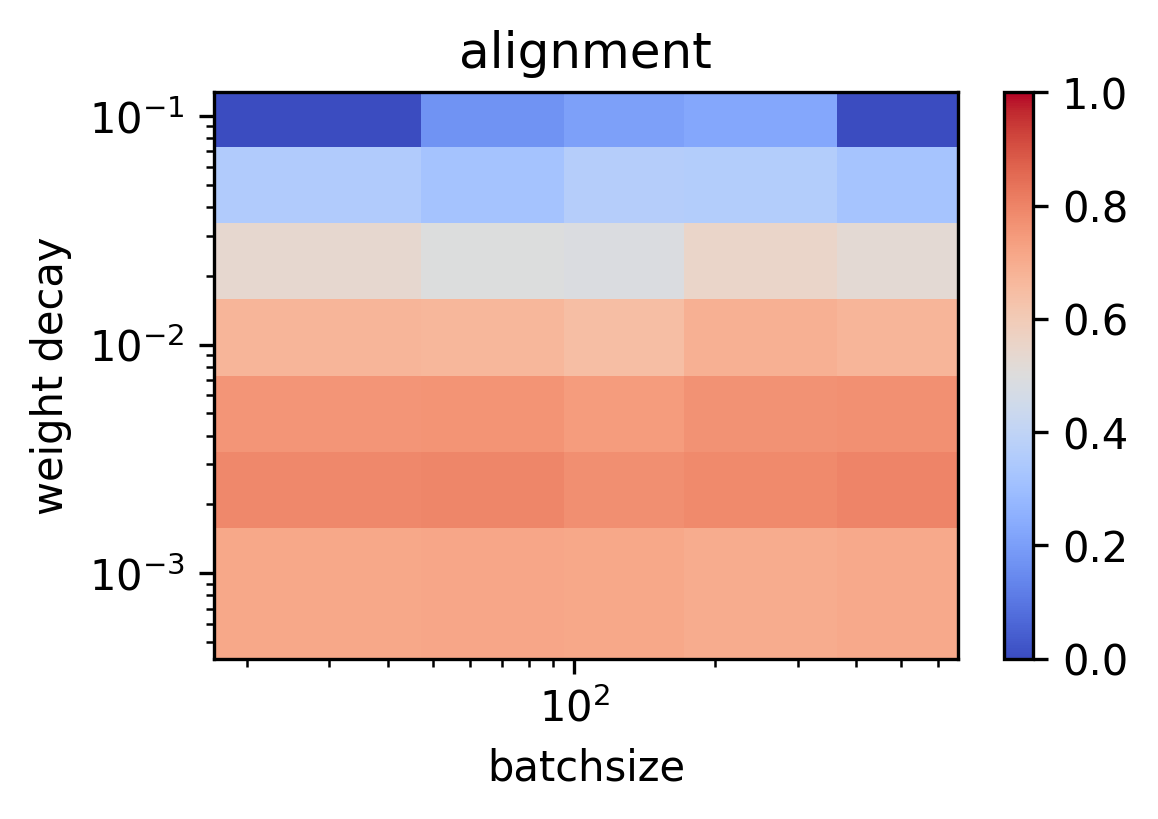} \includegraphics[width=0.4\linewidth]{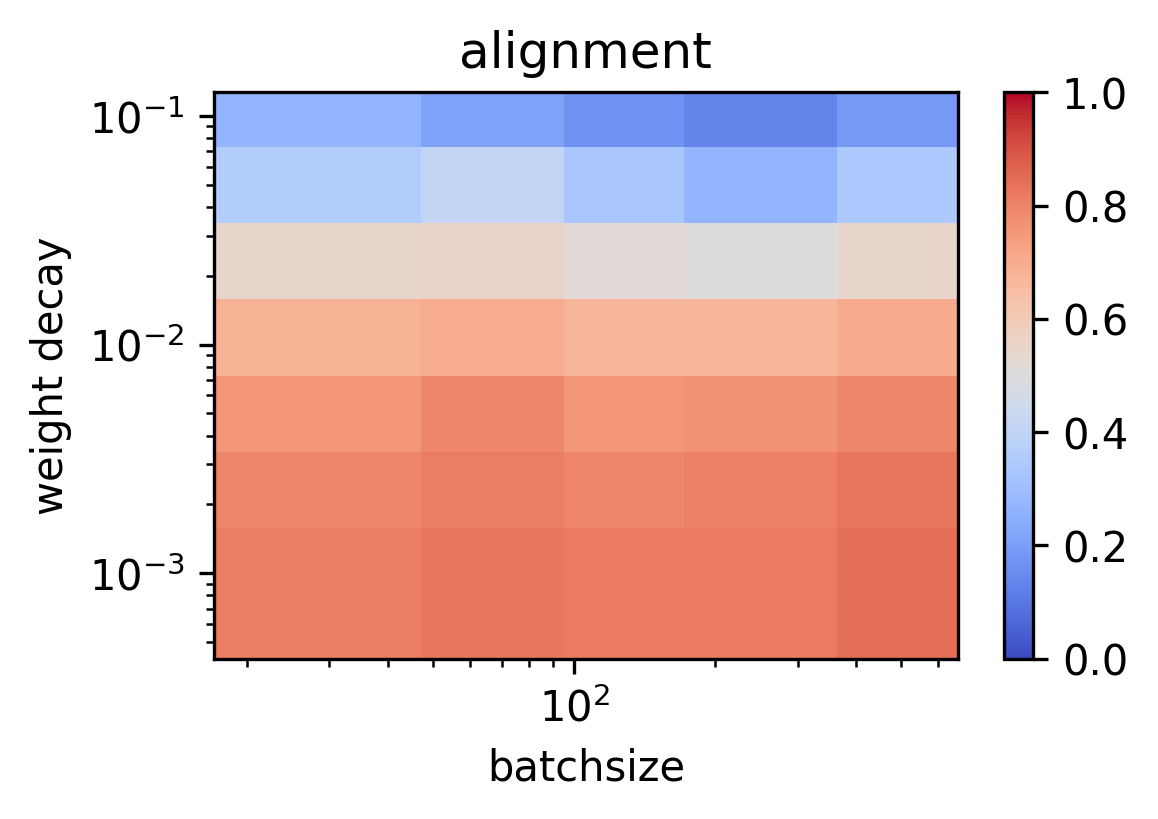}
    \caption{\small The alignment of two $2-$layer ReLU networks independently trained on a teacher-student setting. We measure the alignment for different batchsizes and weight decay. \textbf{Left}: SGD. \textbf{Right}: Adam.
    }
    \vspace{-1em}
    \label{fig:representation alignment2}
\end{figure}

\begin{figure}[t!]
    \centering
    \includegraphics[width=0.4\linewidth]{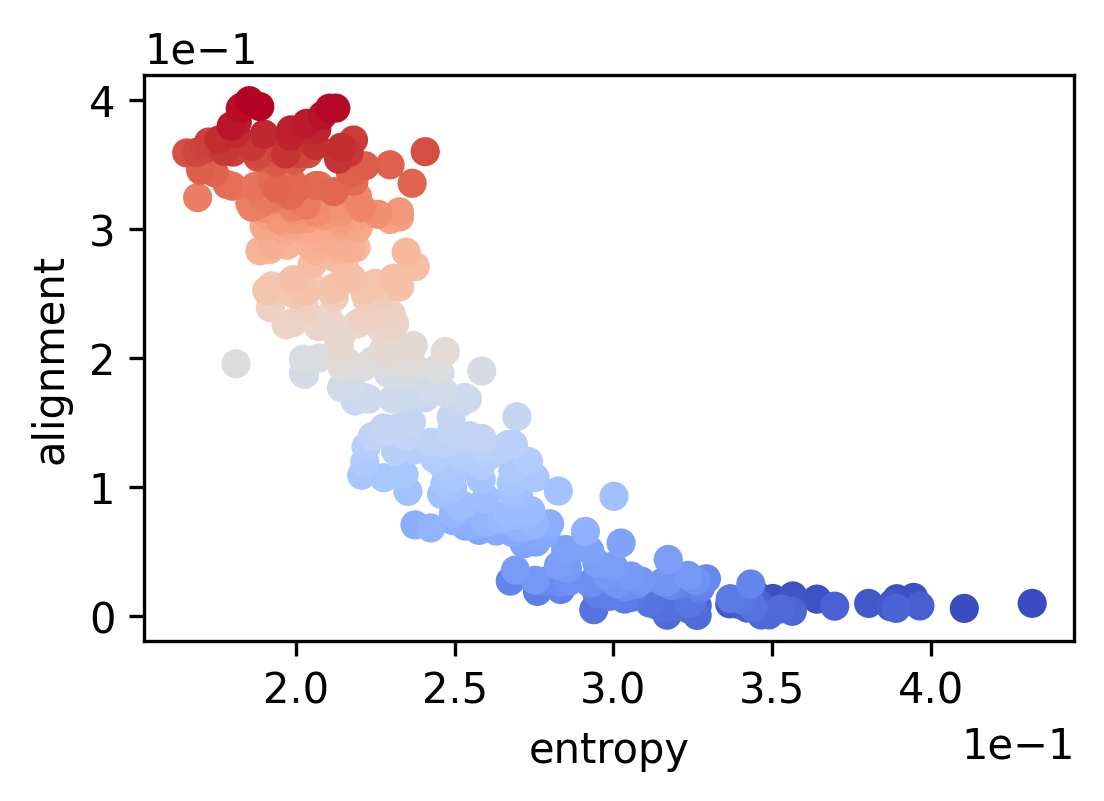} \includegraphics[width=0.4\linewidth]{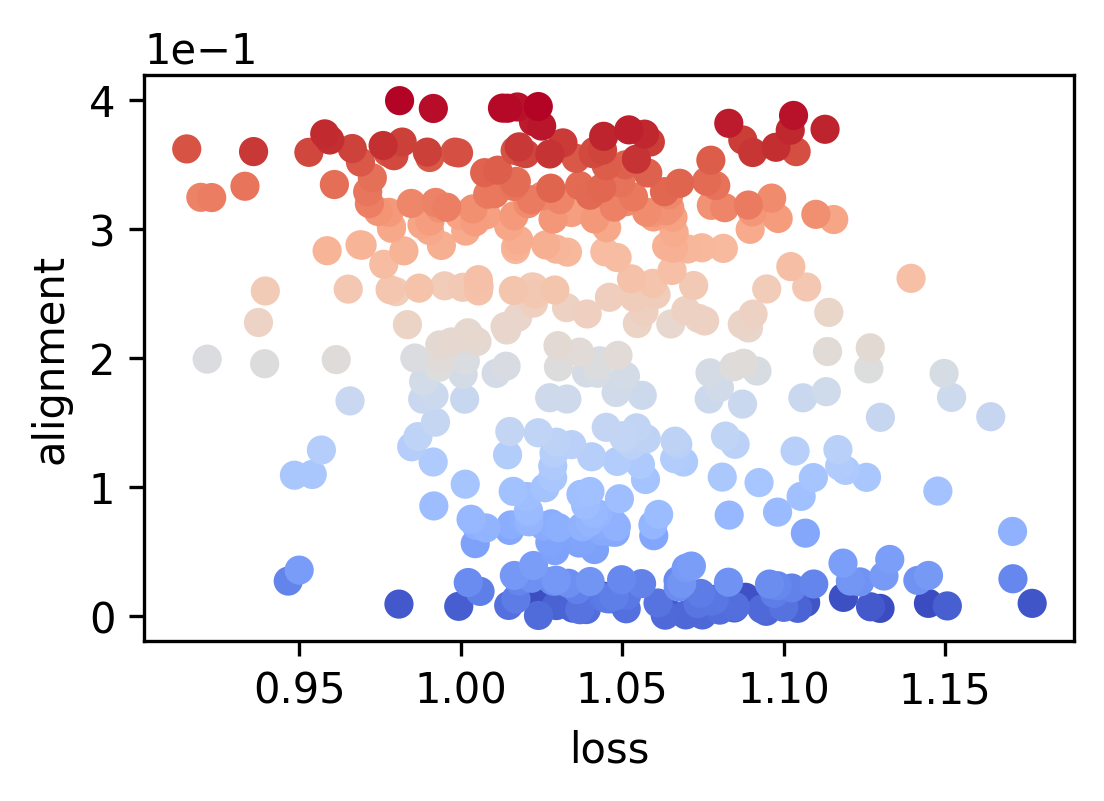}
    \caption{\small Universal alignment is strongly correlated with the entropy but not the loss. The setting is the same as Figure \ref{fig:representation alignment2}, where we use SGD, weight decay $0$ and batchsize $100$.
    }
    \vspace{-1em}
    \label{fig:representation alignment_entropy}
\end{figure}

\begin{figure}[t!]
    \centering
    \includegraphics[width=0.4\linewidth]{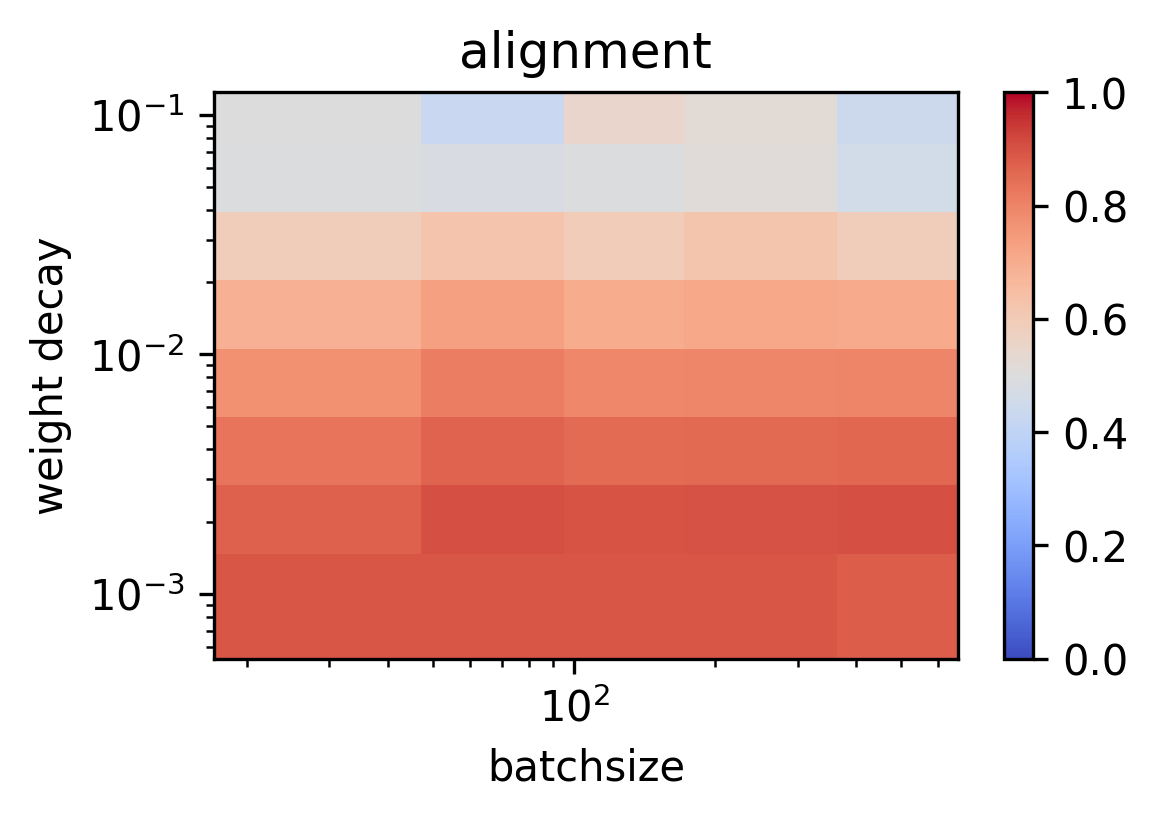} \includegraphics[width=0.4\linewidth]{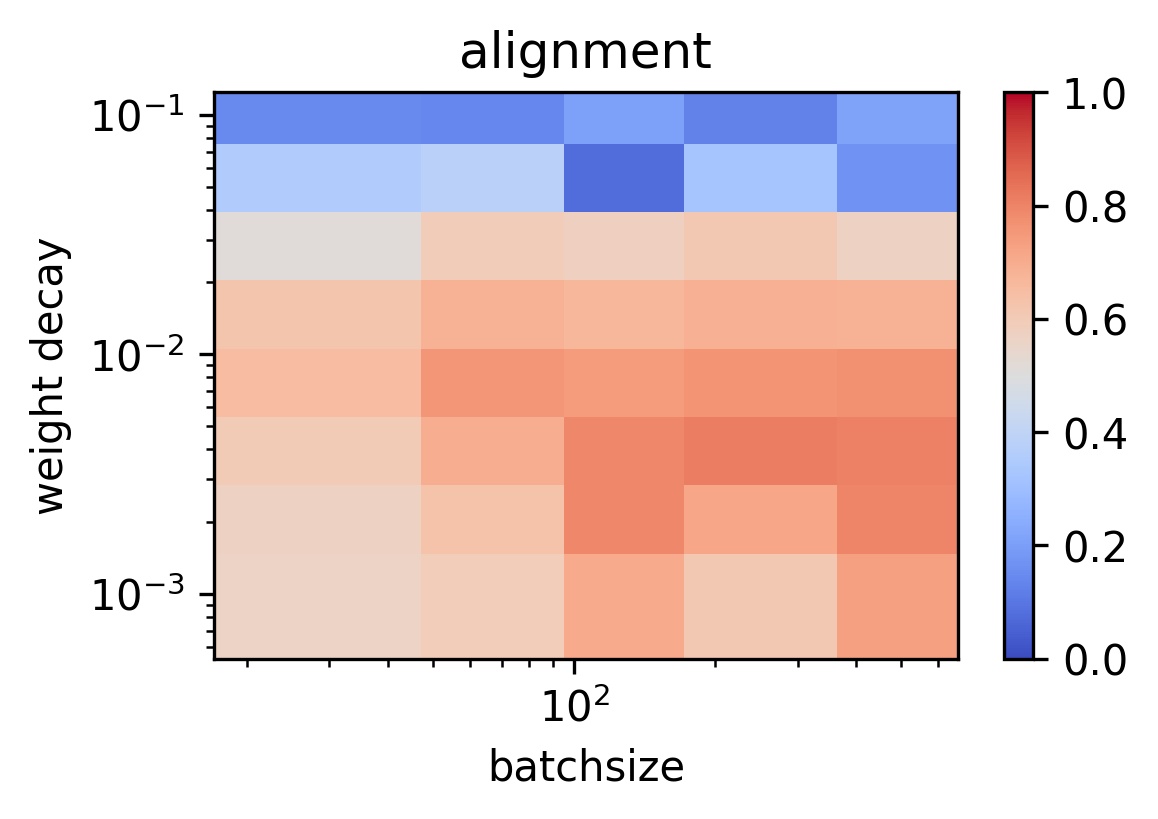}
    \caption{\small The same setting as Figure \ref{fig:representation alignment2}, but for two-layer linear networks. \textbf{Left}: SGD. \textbf{Right}: Adam.
    }
    \vspace{-1em}
    \label{fig:representation alignment3}
\end{figure}

\subsection{Unversal Representation Learning in MLP}
For Figure \ref{fig:representation alignment}, we train two independent 6-layer networks on MNIST. The networks have linear or tanh activation and 128 neurons in each hidden layer, and for the second network, the input MNIST data is transformed by a random Gaussian matrix. We train the networks with Adam optimizer, learning rate $10^{-4}$ for $5$ epochs. During training, we measure the representation alignment between every pair of layers, defined as the cosine similarity between the two sides of \eqref{eq: universal alignment}, averaged over the test set. We then plot the average alignment between the same or different layers of two networks. The input alignment denotes the average alignment between every layer representation and the input data.

In Figures \ref{fig:representation alignment2} and \ref{fig:representation alignment3}, we test the influence of batchsize and weight decay on universal representation in a teacher-student setting. In Figure \ref{fig:representation alignment2}, both the teachers and students are two-layer ReLU networks. In Figure \ref{fig:representation alignment3}, the student is replaced by a linear network. Their hidden dimensions are $100$. Similar to Figure \ref{fig:representation alignment}, we measure the representation alignment between the middle layers of two independently trained networks, and the input of the second network is rotated by a Gaussian matrix. We use random Gaussian data, and the labels are generated by the teacher network. We train the student networks with SGD, learning rate $5\times10^{-2}$ or with Adam, learning rate $10^{-4}$. For both SGD and Adam optimizers, Figures \ref{fig:representation alignment2} and \ref{fig:representation alignment3} suggest that universal alignment does not rely on the batchsize as predicted, but disappears for large weight decay, which verifies Theorem \ref{theo:deep_linear_wd}. In Figure \ref{fig:representation alignment_entropy}, we show that the increase of alignment
is more correlated with the decrease of the entropy rather than the loss.

\clearpage
\begin{figure}[t!]
    \centering
    \includegraphics[width=0.85\linewidth]{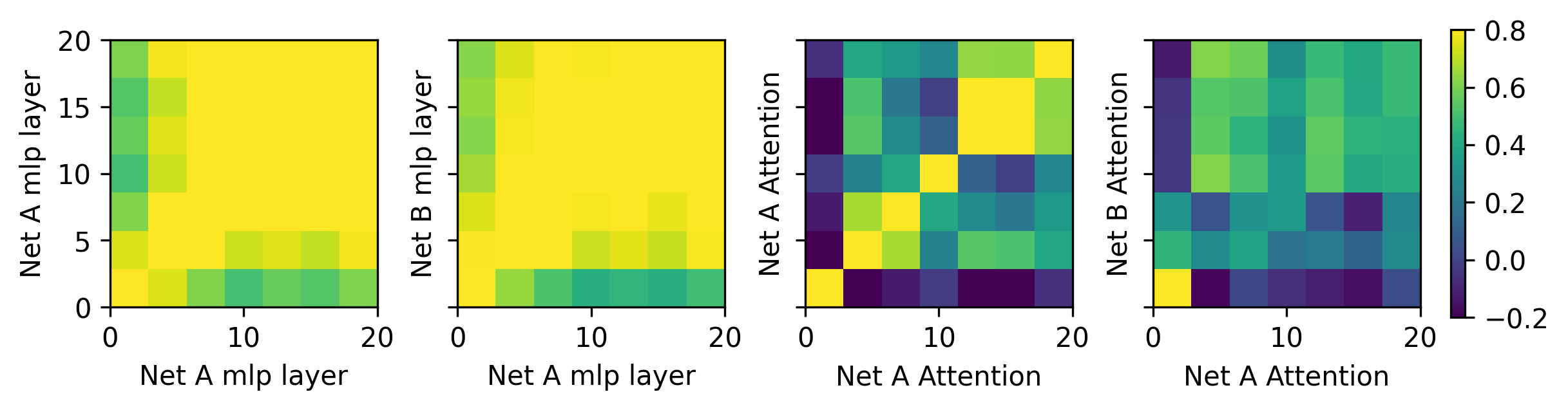}
    \caption{\small Alignment of representations of two larger ViT models pretrained on ImageNet. Net A: ViT-L (\#param: 304M). Net-B: ViT-H (633M). This is similar to Figure~\ref{fig:vit alignment}}
    \label{app fig:vit alignment}
    \vspace{-1em}
\end{figure}
\subsection{Universal Representation in ViT}\label{app sec: vit}
See Figure~\ref{app fig:vit alignment} for the alignment in Vision Transformer. The pretrained weights are taken from \url{https://docs.pytorch.org/vision/main/models.html}. We measure the CKA alignment between the two models or with itself with a minibatch size of 300 images from the ImageNet dataset.

\clearpage

\begin{figure}
    \centering
    \includegraphics[width=0.35\linewidth]{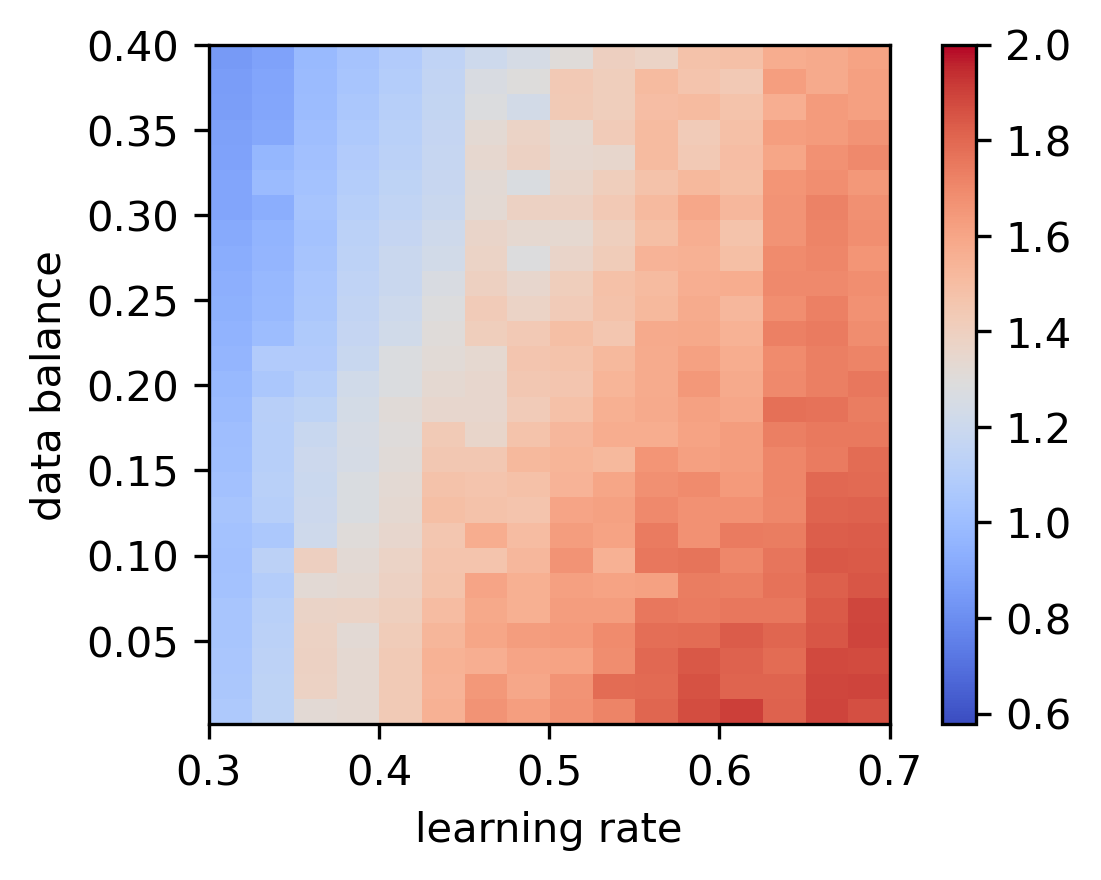}
    \includegraphics[width=0.35\linewidth]{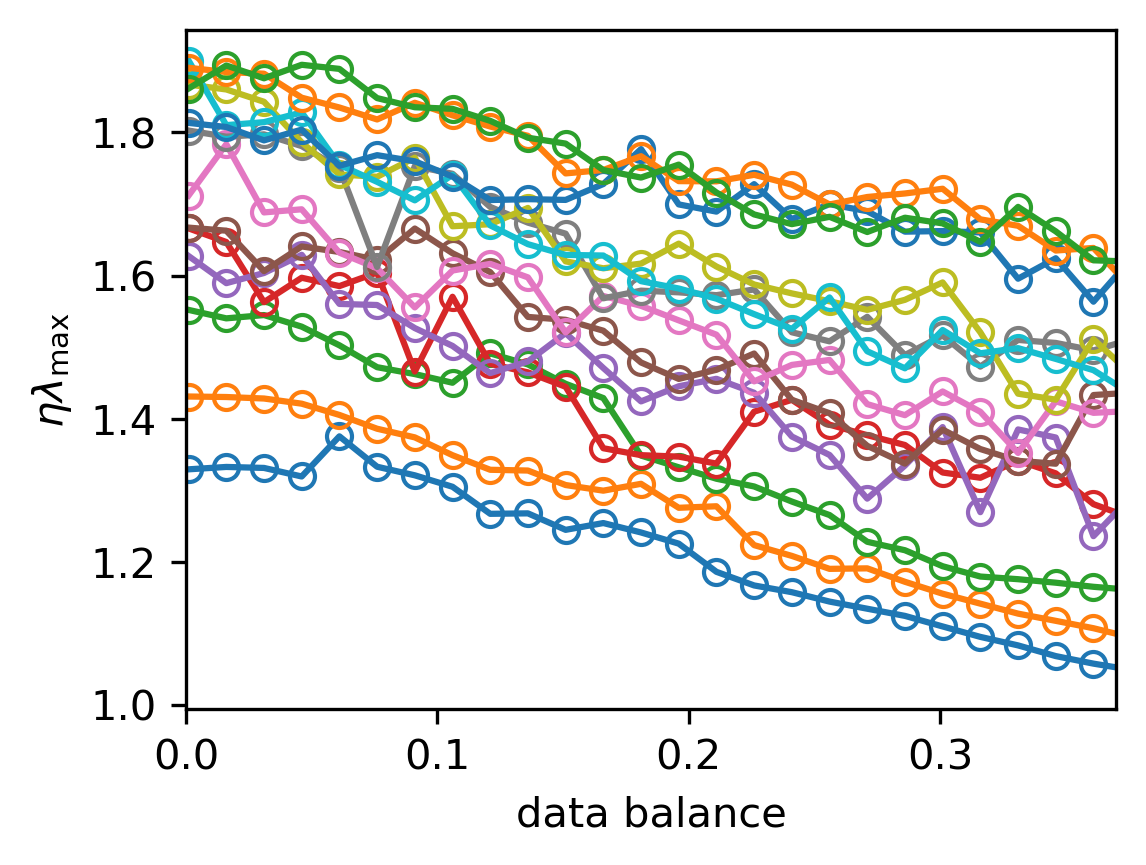}
    \caption{Sharpness at convergence for a two-hidden-layer ReLU network. The setting is identical to that of Figure~\ref{fig:eos phase diagram}. Again, we see that higher imbalance in the label leads to a sharper solution.}
    \label{fig:relu sharpnes}
\end{figure}
\subsection{Edge of Stability}\label{app sec: eos}
See Figure~\ref{fig:eos phase diagram} and \ref{fig:eos dynamics}, where we train a two-layer linear network on a linear regression task with a 2d label $y\in\mathbb{R}^2$. The labels $y= V^*x +\epsilon$, for a ground truth matrix $V^*$ and iid zero-mean noise $\epsilon$ such that $\Sigma_\epsilon={\rm diag}(1, \phi_x)$, where $\phi_x \in (0, 1)$ is called the ``data balance."

We train with different learning rates and data balance. The training proceeds with SGD with a batchsize of 32 for $4\times 10^4$ iterations. Input $x \in \mathbb{R}^2$ is drawn from a standard Gaussian distribution, and the model has dimensions $2\to 10 \to 2$.
 
See Figure~\ref{fig:relu sharpnes} for an experiment with ReLU. Here, the architecture is has dimensions $2\to 10 \to 10 \to 2$.

\clearpage

\section{Theory}
\subsection{Scale Invariance Leads to Flattening}\label{app sec: scale inv}
The symmetry generator is $A=I$. We have
\begin{equation}
\frac{d}{d\lambda}F_{\eta,0}(e^\lambda\theta)=-\eta\mathbb{E}_x(\nabla_\theta\ell(x,\theta)^T\nabla_\theta\ell(x,\theta) = -S<0
\end{equation}
by Theorem \ref{theo:exp symmetry}. As $\frac{d}{d\lambda}F_{\eta,\gamma}(e^\lambda\theta)|_{\lambda=0}=\theta^T\nabla F_{\eta,0}$. Therefore, when we do gradient descent along $-\nabla F_{\eta,0}$, $\lambda$ monotonously increases. Meanwhile, 
\begin{equation}
T(e^\lambda\theta)=\Tr[e^{-2\lambda}\mathbb{E}\nabla^2\ell(x,\theta)]
\end{equation}
decreases with $\lambda$, and thus the sharpness decreases along training.
\subsection{Proof of Theorem~\ref{theo: effective loss}}
\begin{proof}
For notational simplicity we drop the subscript $\gamma$ in the proof. When running gradient descent on $\ell(x,\theta)$, we have
\begin{equation}
    \theta_1=\theta_t-\Lambda \nabla\ell(x,\theta_0).
\end{equation}
When running gradient descent on $\phi_\Lambda$, we have
\begin{equation}
    \theta_1'=\theta_0-\frac{\Lambda}{n}(\nabla \ell(x,\theta_0)+\nabla\phi_{1\Lambda}(x,\theta_0)+\nabla\phi_{2\Lambda}(x,\theta_0))+O(||\Lambda||^3)
\end{equation}
and
\begin{equation}
\begin{aligned}
\theta_{2}'&=\theta_1'-\frac{\Lambda}{n}(\nabla \ell(x,\theta_{1}')+\nabla\phi_1(x,\theta'_{1})+\nabla\phi_2(x,\theta_1'))+O(||\Lambda||^4)\\
&=\theta_0-\frac{2\Lambda}{n}\nabla\ell(x,\theta_0)+\frac{\Lambda}{n^2}\nabla^2\ell(x,\theta_0)\Lambda\nabla\ell(x,\theta_0)-\frac{\Lambda}{n}\nabla\phi_1(x,\theta_0)\\&+\frac{\Lambda}{n^2}\nabla^2\phi_1(x,\theta_0)\Lambda\nabla\ell(x,\theta_0)-\frac{\Lambda}{n}\nabla\phi_2(x,\theta_0)+O(||\Lambda||^4).
\end{aligned}
\end{equation}
Similarly we can obtain
\begin{equation}
\begin{aligned}
\theta_{n}'
&=\theta_0-\Lambda\nabla\ell(x,\theta_0)+\frac{\Lambda}{2}\nabla^2\ell(x,\theta_0)\Lambda\nabla\ell(x,\theta_0)-\Lambda\nabla\phi_1(x,\theta_0)\\&+\frac{\Lambda}{2}\nabla^2\phi_1(x,\theta_0)\Lambda\nabla\ell(x,\theta_0)-\Lambda\nabla\phi_2(x,\theta_0)+O(||\Lambda||^4+||\Lambda||^2/n+||\Lambda||^3/n).
\end{aligned}
\end{equation}
Therefore, we have $\theta_n'=\theta_1+O(||\Lambda||^3+||\Lambda||^2/n)$ if we choose $\phi_1(x,\theta) = \frac{1}{4} \nabla\ell(x,\theta)^T\Lambda\nabla\ell(x,\theta)$.

For small $\nabla^3\ell(x,\theta)$, we have
\begin{equation}
\begin{aligned}
\nabla^2\phi_1(x,\theta_0)\Lambda\nabla\ell(x,\theta_0)&=\nabla^2\ell(x,\theta_0)\Lambda\nabla^2\ell(x,\theta_0)\Lambda\nabla\ell(x,\theta_0)+O(||\Lambda||^2||\nabla^3\ell(x,\theta)||)\\&=\frac{1}{2}\nabla[\nabla\ell(x,\theta_0)^T\Lambda\nabla^2\ell(x,\theta_0)\Lambda\nabla\ell(x,\theta_0)]+O(||\Lambda||^2||\nabla^3\ell(x,\theta)||).
\end{aligned}
\end{equation}
Therefore, we can choose
\begin{equation}
\phi_2(x,\theta):=\frac{1}{2}\nabla\ell(x,\theta)^T\Lambda\nabla^2\ell(x,\theta)\Lambda\nabla\ell(x,\theta)
\end{equation}
to obtain $\theta_n'=\theta_1+O(||\Lambda||^4+||\Lambda||^2/n+||\Lambda||^3/n+||\Lambda||^3||\nabla^3\ell(x,\theta)||)$.
\end{proof}

\subsection{Proof of Theorem~\ref{theo:exp symmetry}}

\begin{proof}

By the definition of the exponential symmetry,
\begin{equation}
\ell(x,e^{\lambda A}\theta)=\ell(x,\theta),\label{eq:symmetry}
\end{equation}
Taking derivative w.r.t. $\lambda$ on \eqref{eq:symmetry}, we have that $\nabla_\theta\ell(x,\theta)^TA\theta=0$. Then taking derivative w.r.t. $\theta$, we have that $A^T\nabla_\theta\ell(x,\theta)+\nabla^2_\theta\ell(x,\theta)A\theta$=0.
 
Let $I(\lambda):=\frac{d}{d\lambda}F_{\eta,\gamma}(e^{\lambda A}\theta^*)$ and $\theta_\lambda:=e^{\lambda A}\theta^*$. Then we have
\begin{equation}
\begin{aligned}
I(\lambda)&=\frac{\eta}{2}(\theta_\lambda)^TA^T\mathbb{E}_\mathcal{B
}[\mathbb{E}_{x\in\mathcal{B}}\nabla^2\ell(x,\theta_\lambda)][\mathbb{E}_{x\in\mathcal{B}}\nabla \ell(x,\theta_\lambda)]+2\gamma(\theta_\lambda)^TA\theta_\lambda\\
&=-\frac{\eta}{2}\mathbb{E}_\mathcal{B}(\mathbb{E}_{x\in\mathcal{B}}\nabla_\theta\ell(x,\theta_\lambda))^TA\mathbb{E}_{x\in\mathcal{B}}\nabla_\theta\ell(x,\theta_\lambda)+2\gamma(\theta_\lambda)^TA\theta_\lambda\\
&=-\frac{\eta}{2}\Tr[\Sigma(\theta_\lambda)A]+2\gamma(\theta_\lambda)^TA\theta_\lambda\\
&=-\frac{\eta}{2}\Tr[\Sigma(\theta_\lambda)\tilde{A}]+2\gamma(\theta_\lambda)^T\tilde{A}\theta_\lambda,
\end{aligned}
\end{equation}
where $\Sigma(\theta_\lambda):=\mathbb{E}_\mathcal{B}[\mathbb{E}_{x\in\mathcal{B}}\nabla_\theta\ell(x,\theta_\lambda)][\mathbb{E}_{x\in\mathcal{B}}\nabla_\theta\ell(x,\theta_\lambda)]^T$ is positive semi-definite. We also use $\Tr[\Sigma(\theta_\lambda)\frac{A-A^T}{2}]=0$ and $(\theta_\lambda)^T\frac{A-A^T}{2}\theta_\lambda=0$. By \cite[Lemma B.1]{ziyin2024parameter}, we have $\Tr[\Sigma(\theta_\lambda)\tilde{A}]=\Tr[e^{-2\lambda}\Sigma(\theta^*)\tilde{A}]$, and thus
\begin{equation}
\begin{aligned}
I(\lambda)&=-\frac{\eta}{2}\Tr[\tilde{A}e^{-2\lambda \tilde{A}}\Sigma(\theta^*)]+2\gamma(\theta^*)^Te^{2\lambda \tilde{A}}\tilde{A}\theta^*\\
&=\sum_i-\frac{\eta}{2}\mu_ie^{-2\lambda\mu_i}(n_i^T\Sigma(\theta^*)n_i)+2\gamma\mu_ie^{2\lambda\mu_i}(n_i^T\theta_i^*)^2,
\end{aligned}
\end{equation}
where $\mu_i,n_i$ are eigenvalues of eigenvectors of the symmetric matrix $\tilde{A}$. Therefore,
\begin{equation}
I'(\lambda)=\sum_i\mu_i^2(\eta e^{-2\lambda\mu_i}n_i^T\Sigma(\theta^*)n_i+4\gamma e^{2\lambda\mu_i}(n_i^T\theta_i^*)^2)\geq0.
\end{equation}
We have $I(\lambda)\equiv0$ (which happens only if $\tilde{A}$ is not full rank) or $I(\lambda)$ strictly monotonic. As $\theta^*$ is a local minimum, we have $I(0)=0$, which gives \eqref{eq:local_minimum}. Then we have $I(\lambda)\equiv0$ or $I(\lambda)=0$ iff $\lambda=0$, which finishes the proof.
\end{proof}

\subsection{Proof of Theorem~\ref{theo: symmetry breaking}}

\begin{proof}

By the definition of $K$-invariance,
and taking derivative $\nabla_\theta$ of both sides of $\ell(x,K(\theta,\lambda)) = \ell(x,\theta)$, we have
\begin{equation}
\nabla_\theta K(\theta,\lambda)^T\nabla_{K(\theta,\lambda)}\ell(x,K(\theta,\lambda))=\nabla_\theta\ell(x,\theta),
\end{equation}
where the l.h.s. follows from the chain rule. This imlies that 
\begin{equation}
    (I + \lambda \nabla Q + O(\lambda^2))\nabla_{K(\theta,\lambda)}\ell(x,K(\theta,\lambda))=\nabla_\theta\ell(x,\theta),
\end{equation}
and so
\begin{equation}
   \nabla_{K(\theta,\lambda)}\ell(x,K(\theta,\lambda))=  (I - \lambda \nabla Q + O(\lambda^2))\nabla_\theta\ell(x,\theta),
\end{equation}

If $F$ is $K$-invariant, the following Equation holds:
\begin{equation}
\ell(x,K(\theta,\lambda))+||\nabla_{K(\theta,\lambda)}\ell(x,K(\theta,\lambda))||^2+||K(\theta,\lambda)||^2=\ell(x,\theta)+||\nabla_\theta\ell(x,\theta)||^2+||\theta||^2.
\label{eq:K-invariant}
\end{equation}

By the assumption, $\ell(x,K(\theta,\lambda))=\ell(x,\theta)$, and \eqref{eq:K-invariant}, we have that 
\begin{align}
||\nabla_\theta\ell(x,\theta)||^2 + ||\theta||^2 & = ||\nabla_{K(\theta,\lambda)}\ell(x,K(\theta,\lambda))||^2 + ||K(\theta,\lambda)||^2\\ 
&= ||\nabla_\theta\ell(x,\theta)||^2 + ||\theta||^2 + 2\lambda ( \nabla^T\ell\nabla Q^T \nabla \ell - \gamma  Q^T\theta)+ O(\lambda^2).
\end{align}
Thus, 
\begin{equation}
    \eta \nabla^T\ell\nabla Q^T \nabla \ell - \gamma  Q^T\theta = 0.
\end{equation}
There are two cases: (1) $\eta \nabla^T\ell\nabla Q^T \nabla \ell - \gamma Q^T\theta=0 $, and (2) $\eta \nabla^T\ell\nabla Q^T \nabla \ell - \gamma Q^T\theta \neq 0 $. 

For case (1), we are done. For case (2), the equation cannot hold for $\gamma + d \gamma$ because the first term is independent of $\gamma$. Thus, we can only have case (1).

This means that for any $\theta$
\begin{equation}
    \| K(\theta, \lambda) \|^2 = \|\theta\|^2 + O(\lambda^2),
\end{equation}
which is only possible if $\| K(\theta, \lambda) \|^2 = \|\theta\|^2$. This completes the proof.
\end{proof}

\subsection{Proof of Theorem \ref{theo:orthogonal}}
\begin{proof}
By definition we have $||O\theta||^2=||\theta||^2$. Take derivative on both sides of $\ell(x,O\theta)=\ell(x,\theta)$, we have
\begin{equation}
O^T\nabla_{O\theta}\ell(x,O\theta)=\nabla_\theta\ell(x,\theta).
\end{equation}
Thus we have
\begin{equation}
\nabla_{O\theta}\ell(x,O\theta)=O^{-T}\nabla_\theta\ell(x,\theta),
\end{equation}
which gives $||\nabla_{O\theta}\ell(x,O\theta)||^2=||O^{-T}\nabla_\theta\ell(x,\theta)||^2=||\nabla_\theta\ell(x,\theta)||^2$. Combining the above results we have $F_{\eta,\gamma}(O\theta)=F_{\eta,\gamma}(\theta)$.
\end{proof}

\subsection{Proof of Theorem \ref{theo:layer_balance}}
\begin{proof}
    Rescaling symmetry implies that if we make the transform
    \begin{equation}
        W_i \to e^{\lambda}W_i,\ W_j \to e^{-\lambda}W_j,
    \end{equation}
    $L$ does not change.

    This corresponds to the choice of 
    \begin{equation}
        A_{klm}^{\tilde{k}\tilde{l}\tilde{m}}=\begin{cases}
1 & \ k=\tilde{k}=i \\
-1 & \ k=\tilde{k}=j\\
0&\ \text{otherwise}
\end{cases}
    \end{equation}in Theorem \ref{theo:exp symmetry}, where the index $klm$ corresponds to the $m-$th element of the $l-$th unit of the $k-$th layer. Then we have
    \begin{equation}
    \eta(\mathbb{E}\Tr[g_ig_i^T-g_jg_j^T])=4\gamma(\Tr[W_iW_i^T-W_jW_j^T]).
    \end{equation}
    This finishes the proof by setting $\gamma=0$.
\end{proof}

\subsection{Proof of Theorem~\ref{theo: neuron balance}}
\begin{proof}
We can choose $A$ to be a rescaling matrix w.r.t. the $j-$th neuron of the $i-$th layer. Specifically, we choose     \begin{equation}
        A_{klm}^{\tilde{k}\tilde{l}\tilde{m}}=\begin{cases}
1 & \ k=\tilde{k}=i,l=\tilde{l}=j \\
-1 & \ k=\tilde{k}=i+1,m=\tilde{m}=j\\
0&\ \text{otherwise}
\end{cases}
    \end{equation} in Theorem \ref{theo:exp symmetry}, which gives
\begin{equation}
\eta\mathbb{E}\Tr[g_{i,j,:}g_{i,j,:}^\top-g_{i+1,:, j}g_{i+1,:,j}^\top]=4\gamma\Tr[w_{i,j,:}w_{i,j,:}^T-w_{i+1,:,j}w_{i+1,:,j}^T].
\end{equation}
\end{proof}

\subsection{Gradient Imbalance in Polynomial Networks}\label{app sec: polynomial net}

\begin{theorem}\label{theo: polynomial net balance}
(Neuron Balance) For all local minima of 
Eq.~\eqref{eq: free energy} and any $i,\ j$, 
\begin{equation}
\eta\mathbb{E}\Tr[g_{i,j,:}g_{i,j,:}^\top-dg_{i+1,:, j}g_{i+1,:,j}^\top]=4\gamma\Tr[w_{i,j,:}w_{i,j,:}^T-dw_{i+1,:,j}w_{i+1,:,j}^T].
\end{equation}
\end{theorem}

This means that unless $d=1$, these networks will either have a noise or weight explosion problem. If $\gamma=0$, the gradient fluctuation grows like $d^D$, exponential in depth $D$. When $d<1$, later layers will have an exploding noise; when $d>1$, earlier layers will have an exploding noise. When both $\eta$ and $\gamma \neq 0$, the sum of the noise and gradient norm will explode exponentially. In some sense, this implies that linear or sublinear types of activations are the only stable activations for deep neural networks.

\begin{proof}
We can still choose $A$ to be a rescaling matrix, but this time we should rescale the $i+1-$th layer more
   \begin{equation}
        w_{i,j,:} \to e^{\lambda}w_{i,j,:},\ w_{i+1,:,j} \to e^{-d\lambda}w_{i+1,:,j}.
    \end{equation}
This corresponds to    
\begin{equation}
        A_{klm}^{\tilde{k}\tilde{l}\tilde{m}}=\begin{cases}
1 & \ k=\tilde{k}=i,l=\tilde{l}=j \\
-d & \ k=\tilde{k}=i+1,m=\tilde{m}=j\\
0&\ \text{otherwise}
\end{cases}
    \end{equation} in Theorem \ref{theo:exp symmetry}, which gives
\begin{equation}
\eta\mathbb{E}\Tr[g_{i,j,:}g_{i,j,:}^\top-dg_{i+1,:, j}g_{i+1,:,j}^\top]=4\gamma\Tr[w_{i,j,:}w_{i,j,:}^T-dw_{i+1,:,j}w_{i+1,:,j}^T].
\end{equation}
This gives
\begin{equation}
 \mathbb{E}\Tr[g_{i,j,:}g_{i,j,:}^\top] = d\mathbb{E}\Tr[g_{i+1,:, j}g_{i+1,:,j}^\top]
\end{equation} by choosing $\gamma=0$.
\end{proof}

\subsection{Proof of Theorem~\ref{theo:attention}}

\begin{proof}
The double rotation symmetry can be written as
\begin{equation}
U\to e^{\lambda A}U,\ W\to e^{-\lambda A}W,
\end{equation}
where $A$ is an arbitrary matrix. We can thus choose the following generator
\begin{equation}
    A_{klm}^{\tilde{k}\tilde{l}\tilde{m}}=\begin{cases}
1 & \ k=\tilde{k}=1,\ l=i,\tilde{l}=j,\ \text{or}\ l=j,\tilde{l}=i \\
-1& \ k=\tilde{k}=2,\ l=i,\tilde{l}=j,\ \text{or}\ l=j,\tilde{l}=i\\
0&\ \text{otherwise}
\end{cases}
\end{equation} in Theorem \ref{theo:exp symmetry}, where $k=1$ corresponds to $W$ and $k=2$ corresponds to $U$. This gives
\begin{equation}
\sum_k\eta [G_{W_{ki}}G_{W_{kj}}-G_{U_{ik}}G_{U_{jk}}]=\sum_k4\gamma[W_{ki}W_{kj}-U_{ik}U_{jk}],
\end{equation}
which finishes the proof.
\end{proof}

\subsection{Proof of Lemma \ref{lemma:sharpness}}
\begin{proof}
By definition, we have
\begin{equation}
   e^{\lambda A} \nabla_{e^{\lambda A}\theta }^2 \ell (x, e^{\lambda A}\theta ) e^{\lambda A}  = \nabla^2 \ell(x, \theta),
\end{equation}
and thus
\begin{equation}
T(e^{\lambda A}\theta)=\Tr[e^{-2\lambda A}\E\nabla^2\ell(x,\theta)].
\end{equation}
Let $A:=\sum_i\mu_in_in_i^T$, and thus
\begin{equation}
T(e^{\lambda A}\theta)=\sum_{i}e^{-2\lambda\mu_i}(n_i^T\E\nabla^2\ell(x,\theta)n_i).
\end{equation}
As $A\E\nabla^2\ell(x,\theta)\neq0$, there exists $i$ such that $\mu_i\neq0$ and  $n_i^T\E\nabla^2\ell(x,\theta)n_i\neq0$.  Therefore, we have $\lim_{\lambda\to+\infty}|T(e^{\lambda A}\theta)|=+\infty$ if $\mu_i<0$, and $\lim_{\lambda\to-\infty}|T(e^{\lambda A}\theta)|=+\infty$ if $\mu_i>0$.
\end{proof}


\subsection{Proof of Theorem~\ref{theo: universal representation}}

We first prove the following theorem, which we will leverage to prove Theorem~\ref{theo: universal representation}.

\begin{theorem}
\label{theo:deep_linear}
Let $V'=\sqrt{\Sigma_\epsilon}V\sqrt{\Sigma_x}$ such that $V'=\tilde{U}S'\tilde{V}$ is its SVD and $\text{rank}(V')=d$. Assume that every layer has more than $d$ hidden units. Then if $\gamma=0$ and $\eta=0^+$, at any global minimum of \eqref{eq: free energy}, we have
\begin{equation}
\sqrt{\Sigma_\epsilon}M_1W_{D}=\tilde{U}\Sigma_DU_{D-1}^T,\ W_i=U_i\Sigma_iU_{i-1}^T,\ W_1M_2M_3\sqrt{\Sigma_x}=U_1\Sigma_1\tilde{V},
\end{equation}
for $i=2,\cdots,D-1$, where $U_i$ are arbitrary matrices satisfying $U_i^TU_i=I_{d\times d}$, and $\Sigma_x=\mathbb{E}[xx^T]$, $\Sigma_\epsilon=\mathbb{E}[\epsilon\epsilon^T]$. Moreover,
\begin{equation}
\begin{aligned}
&\Sigma_1=(\Tr S')^{-\frac{D-2}{2D}}\frac{\Tr[M_2M_3\Sigma_xM_3^TM_2^T]^{\frac{D-1}{2D}}}{\Tr[M_1^T\Sigma_\epsilon M_1]^{\frac{1}{2D}}}\sqrt{S'},\\
&\Sigma_D=(\Tr S')^{-\frac{D-2}{2D}}\frac{\Tr[M_1^T\Sigma_\epsilon M_1]^{\frac{D-1}{2D}}}{\Tr[M_2M_3\Sigma_xM_3^TM_2^T]^{\frac{1}{2D}}}\sqrt{S'},\\
&\Sigma_i=(\Tr S')^{1/D}(\Tr[M_1^T\Sigma_\epsilon M_1]\Tr[M_2M_3\Sigma_xM_3^TM_2^T])^{-\frac{1}{2D}}I_d.
\end{aligned}
\end{equation}
\end{theorem}

\begin{proof}
Consider two consecutive layers $W_i$ and $W_{i+1}$. Using Theorem \ref{theo:attention}, we have
\begin{equation}
\eta\E[G_{W_{i+1}}^T G_{W_{i+1}}-G_{W_i} G_{W_i}^T]=0.
\end{equation}
By the MSE loss $\ell(x,y)=||y-M_1W_D\cdots W_1M_2M_3x||^2$, this gives
\begin{equation}
W_ih_i\mathbb{E}[||\xi_{i+1}^T\tilde{r}||^2\tilde{x}\tilde{x}^T]h_i^TW_i^T=W_{i+1}^T\xi_{i+1}^T\mathbb{E}[||h_i\tilde{x}||^2\tilde{r}\tilde{r}^T]\xi_{i+1}W_{i+1},
\end{equation}
where $\tilde{x}=\E_{x\in\mathcal{B}}x$ and $\tilde{r}:=\E_{x\in\mathcal{B}}[y-M_1W_D\cdots W_1M_2M_3x]$ satisfy $\E\tilde{x}=\E\tilde{r}=0$, and thus $\E\tilde{x}\tilde{x}^T=\frac{\Sigma_x}{|\mathcal{B}|}$ and $\E\tilde{\epsilon}\tilde{\epsilon}^T=\frac{\Sigma_\epsilon}{|\mathcal{B}|}$. We use the fact that at the global minimum we have $M_1W_D\cdots W_1M_2M_3=V$, and thus $y-M_1W_D\cdots W_1M_2M_3x$ is independent of $x$. We denote $\xi_{i+1}:=M_1W_D\cdots W_{i+2}$, $h_i:=W_{i-1}\cdots W_1M_2M_3$ for $i=2,3,\cdots,D-2$, and $\xi_D:=M_1$, $h_1:=M_2M_3$.

Finally denote $W_1'=W_1M_2M_3\sqrt{\Sigma_x}$ and $W_D'=\sqrt{\Sigma_\epsilon}M_1W_D$, which gives $W_D'\cdots W_1'=V'$ and 
\begin{equation}
W_{i+1}^{\top} \frac{W_{i+2}^{\top} \cdots W_D^{\prime \top} W_D^{\prime} \cdots W_{i+2}}{\Tr \left[W_{i+2}^{\top} \cdots W_D^{\prime \top} W_D^{\prime} \cdots W_{i+2}\right]} W_{i+1}=W_i \frac{W_{i-1} \cdots W_1^{\prime} W_1^{\prime \top} \cdots W_{i-1}^{\top}}{\Tr \left[W_{i-1} \cdots W_1^{\prime} W_1^{\prime \top} \cdots W_{i-1}^{\top}\right]} W_i^{\top}
\label{eq:eqi}
\end{equation}
for $i=2,3,\cdots,D-2$. For $i=1$ we have
\begin{equation}
W_{2}^{\top} \frac{W_{3}^{\top} \cdots W_D^{\prime \top} W_D^{\prime} \cdots W_{3}}{\Tr \left[W_{3}^{\top} \cdots W_D^{\prime \top} W_D^{\prime} \cdots W_{3}\right]} W_{2}= \frac{W_1^{\prime} W_1^{\prime \top} }{\Tr \left[M_2M_3\Sigma_x M_2^\top M_3^\top\right]} 
\label{eq:eq1}
\end{equation}
and for $i=D-1$ we have
\begin{equation}
\frac{W_D^{\prime \top} W_D^{\prime}}{\Tr \left[M_1^T\Sigma_\epsilon M_1\right]}=W_{D-1} \frac{W_{D-2} \cdots W_1^{\prime} W_1^{\prime \top} \cdots W_{D-2}^{\top}}{\Tr \left[W_{D-2} \cdots W_1^{\prime} W_1^{\prime \top} \cdots W_{D-2}^{\top}\right]} W_{D-1}^{\top}.
\label{eq:eqD}
\end{equation}
Lemma \ref{lemma:decomposition} proves that we can decompose the matrices $W_1',W_2,\cdots,W_{D-1},W_D'$ as
\begin{equation}
W_D'=U_D\Sigma_DU_{D-1}^T,\ W_{D-1}=U_{D-1}\Sigma_{D-1}U_{D-2}^T,\cdots,\ W_1'=U_1\Sigma_1U_0.
\end{equation}

By minimizing \eqref{eq: free energy} at $\eta=0^+$, we need to minimize
\begin{equation}
\mathbb{E}||y-W_D\cdots W_1x||^2=\mathbb{E}||\epsilon||^2+||(V-W_D\cdots W_1)\sqrt{\Sigma_x}||^2.
\end{equation}
Thus we have
\begin{equation}
(V-W_D\cdots W_1)\sqrt{\Sigma_x}=0,
\end{equation}
which gives
\begin{equation}
W_D'\cdots W_1'=\sqrt{\Sigma_\epsilon}V\sqrt{\Sigma_x}=V'.
\end{equation}
Then we obtain $U_D=\tilde{U}$, $U_0=\tilde{V}$ and
\begin{equation}
\Sigma_D\Sigma_{D-1}\cdots\Sigma_1=S'.
\label{eq:S'}
\end{equation}
We can assume $\Sigma_D,\cdots,\Sigma_1\in\mathbb{R}^{d\times d}$ because their ranks are the same by \eqref{eq:eqi}, \eqref{eq:eq1} and \eqref{eq:eqD}.

\eqref{eq:eqi} gives
\begin{equation}
\frac{\Sigma_{i+1}^2\Sigma_{i+2}^2\cdots\Sigma_D^2}{\Tr [\Sigma_{i+2}^2\cdots\Sigma_D^2]}=\frac{\Sigma_1^2\cdots\Sigma_{i-1}^2\Sigma_i^2}{\Tr[\Sigma_1^2\cdots\Sigma_{i-1}^2]},
\end{equation}
and thus $\Sigma_i=cI_d$ for $i=2,3,\cdots,D-2$ and
\begin{equation}
\frac{\Sigma_1^2}{\Tr[\Sigma_1^2]}=\frac{\Sigma_D^2}{\Tr[\Sigma_D^2]}.
\label{eq:Sigma1D_1}
\end{equation}
\eqref{eq:eq1} and \eqref{eq:eqD} give
\begin{equation}
\frac{\Sigma_{2}^2\Sigma_{3}^2\cdots\Sigma_D^2}{\Tr [\Sigma_{3}^2\cdots\Sigma_D^2]}=\frac{\Sigma_1^2}{\Tr[M_2M_3\Sigma_xM_3^TM_2^T]},\ \frac{\Sigma_D^2}{\Tr[M_1^T\Sigma_\epsilon M_1]}=\frac{\Sigma_1^2\cdots\Sigma_{D-1}^2}{\Tr[\Sigma_1^2\cdots\Sigma_{D-2}^2]},
\end{equation}
and thus
\begin{equation}
c^2\frac{\Sigma_D^2}{\Tr[\Sigma_D^2]}=\frac{\Sigma_1^2}{\Tr[M_2M_3\Sigma_xM_3^TM_2^T]},\ \frac{\Sigma_D^2}{\Tr[M_1^T\Sigma_\epsilon M_1]}=c^2\frac{\Sigma_1^2}{\Tr[\Sigma_1^2]}.
\label{eq:Sigma1D_2}
\end{equation}
Combining \eqref{eq:S'}, \eqref{eq:Sigma1D_1} and \eqref{eq:Sigma1D_2}, we finish the proof.

\end{proof}

Now, we are ready to prove Theorem~\ref{theo: universal representation}.
\begin{proof}
By Theorem \ref{theo:deep_linear}, at any global minimum of \eqref{eq: free energy}, the solution of a $D_A$-layer network for the dataset $\mathcal{D}_M$ is given by
\begin{equation}
\sqrt{\Sigma_\epsilon}M_1W^A_{D_A}=\tilde{U}\Sigma_DU_{D_A-1}^T,\ W^A_i=U_i\Sigma_iU_{i-1}^T,\ W^A_1M_2M_3\sqrt{\Sigma_x}=U_1\Sigma_1\tilde{V}
\end{equation}
for $i=2,\cdots,D-1$, where $U_i$ are arbitrary matrices satisfying $U_i^TU_i=I_{d\times d}$, and
\begin{equation}
\Sigma_1\propto\sqrt{S'},\ \Sigma_D\propto\sqrt{S'},\ \Sigma_i\propto I_d
\end{equation}
for some constants $c_1,c_2,c_3$. The solution suggests that
\begin{equation}
h_A^{L_A}(x)=\Pi_{i=1}^{L_A}W_i^AM_2M_3x\propto U_{L_A}\sqrt{S'}\tilde{V}\Sigma_x^{-1/2}x.
\label{eq:solution_A}
\end{equation}
Similarly
\begin{equation}
h_B^{L_B}(x)=\Pi_{i=1}^{L_B}W_i^BM_2'M_3'x\propto U_{L_B}\sqrt{S'}\tilde{V}\Sigma_x^{-1/2}x.
\label{eq:solution_B}
\end{equation}

The proof is complete by comparing \eqref{eq:solution_A} and \eqref{eq:solution_B}.
\end{proof}

One might also consider the case where the minimal width of network B is $d_B<d$. In this case, we denote $\bar{S}'\in\mathbb{R}^{d_B\times d_B}$ containing top $d_B$ values of $S'$. Then we have
\begin{equation}
h_B^{L_B}(x)=\Pi_{i=1}^{L_B}W_i^BM_2'M_3'\propto U_{L_B}\sqrt{\bar{S}'}\tilde{V}\Sigma_x^{-1/2}x.
\end{equation}
It is now not fully aligned with $h_A(x)$. To calculate the alignment, the corresponding kernels are
\begin{equation}
K_A(x_1,x_2)=h_A^{L_A}(x_1)^Th_A^{L_A}(x_2)=c_1x_1^T\Sigma_x^{-1/2}\tilde{V}^TS'\tilde{V}\Sigma_x^{-1/2}x_2
\end{equation}
and
\begin{equation}
K_B(x_1,x_2)=h_B^{L_B}(x_1)^Th_B^{L_B}(x_2)=c_2 x_1^T\Sigma_x^{-1/2}\tilde{V}^T\bar{S}'\tilde{V}\Sigma_x^{-1/2}x_2
\end{equation}
for some constants $c_1,c_2>0$. We then have
\begin{equation}
\langle K_A,K_A\rangle_F=\mathbb{E}K_A(x_1,x_2)^2=c_1^2\Tr[(S')^2].
\end{equation}
Similarly we have
\begin{equation}
\langle K_B,K_B\rangle_F=c_2^2\Tr[(\bar{S}')^2]
\end{equation}
and
\begin{equation}
\langle K_A,K_B\rangle_F=\mathbb{E}K_A(x_1,x_2)K_B(x_1,x_2)=c_1c_2\Tr[\bar{S}'S'].
\end{equation}
Therefore, the alignment is given by
\begin{equation}
\frac{\langle K_A,K_B\rangle_F}{\sqrt{\langle K_A,K_A\rangle_F\langle K_B,K_B\rangle_F}}=\frac{\Tr[\bar{S}'S']}{\sqrt{\Tr[(S')^2]\Tr[(\bar{S}')^2]}}=\sqrt{\frac{\Tr[(\bar{S}')^2]}{\Tr[(S')^2]}},
\end{equation}
which is some value between $0$ and $1$.

In the end of this section we proves the following technical lemma.
\begin{lemma}
\label{lemma:decomposition}
Suppose that matrices $W_1,W_2,\cdots,W_D$ satisfy
\begin{equation}
W_{i+1}^TW_{i+2}^T...W_D^TW_D...W_{i+2}W_{i+1}=\lambda_iW_iW_{i-1}...W_1W_1^T...W_{i-1}^TW_i
\end{equation}
for some $\lambda_i>0$ and $i=1,2,\cdots,D-1$, then we have write the SVD of $W_1,W_2,\cdots,W_D$ as
\begin{equation}
W_D=U_D\Sigma_DU_{D-1}^T,\ W_{D-1}=U_{D-1}\Sigma_{D-1}U_{D-2}^T,\cdots,\ W_1=U_1\Sigma_1U_0,
\end{equation}
where $\Sigma_D,\cdots,\Sigma_1$ are the singular values.
\end{lemma}
\begin{proof}
Denote $P_i:=(W_i\cdots W_1)(W_i\cdots W_1)^T$. We then have
\begin{equation}
W_{i+1}^TW_{i+2}^T...W_D^TW_D...W_{i+2}W_{i+1}=\lambda_iP_i
\end{equation}
and
\begin{equation}
W_{i+2}^T...W_D^TW_D...W_{i+2}=\lambda_{i+1}W_{i+1}P_iW_i^T,
\end{equation}
which gives
\begin{equation}
P_i=\frac{\lambda_{i+1}}{\lambda_i}S_{i+1}P_iS_{i+1},
\end{equation}
where $S_{i+1}:=W_{i+1}^TW_{i+1}.$

Suppose that $S_{i+1}=V\Lambda V^T$, and thus we have
\begin{equation}
A=c\Lambda A\Lambda,
\end{equation}
where $c:=\frac{\lambda_{i+1}}{\lambda_i}$ and $A:=V^TP_iV$. The $(j,k)$ element gives
\begin{equation}
A_{jk}(1-c\lambda_j\lambda_k)=0.
\end{equation}
If $A_{jk}\neq0$, as $A$ is semi-definite, we also have $A_{jj},A_{kk}\neq0$, which gives $c\lambda_i^2=c\lambda_j^2=1$. Therefore, we have $A\Lambda=\Lambda A$, and thus
\begin{equation}
P_iS_{i+1}=S_{i+1}P_i.
\end{equation}
This shows that $P_i$ and $S_{i+1}$ share the same eigenspace. Moreover, we have $P_i=W_iP_{i-1}W_i^T$. Denote $W_i=U_i\Sigma_iU_{i-1}^T$ and $W_{i+1}=U_{i+1}\Sigma_{i+1}V$. As $P_{i-1}$ share the same eigenspace with $W_i^TW_i$, we have $P_{i-1}=U_{i-1}\Lambda U_{i-1}^T$ for some diagonal matrix $\Lambda$, which gives $P_i=U_i\Sigma_i\Lambda\Sigma_i U_i$. As $P_i$ share the same eigenspace with $W_{i+1}^TW_{i+1}=V^T\Sigma_{i+1}^2V$, we obtain $V=U_i$, which finishes the proof.
\end{proof}

\subsection{Proof of Theorem \ref{theo:deep_linear_wd}}
\begin{proof}
By Theorem \ref{theo:attention} we have
\begin{equation}
W_{i+1}^TW_{i+1}=W_iW_i^T
\end{equation}
for $i=1,\cdots,D-1$, which suggests that $W_i=U_i\Sigma_i U_{i-1}^T$ with $\Sigma_i^2=\Sigma_{i+1}^2$. At the global minimum we have $U_D(\Pi_{i=1}^D\Sigma_i) U_0^T=\Pi_{i=1}^DW_i=M_1^{-1}VM_2^{-1}$, which shows that the left side is the SVD of $M_1^{-1}VM_2^{-1}$. This finishes the proof.
\end{proof}

\subsection{Proof of Theorem~\ref{theo: sharpness}}
\begin{proof}
By Theorem \ref{theo:deep_linear} we have
\begin{equation}
\sqrt{\Sigma_\epsilon} U=\frac{\Tr[\Sigma_\epsilon]^{\frac{1}{4}}}{\Tr[\Sigma_x]^{\frac{1}{4}}}\tilde{U}\sqrt{S'}U_1^T,\ W\sqrt{\Sigma_x}=\frac{\Tr[\Sigma_x]^{\frac{1}{4}}}{\Tr[\Sigma_\epsilon]^{\frac{1}{4}}}U_1\sqrt{S'}\tilde{V}.
\end{equation}
By \cite[Proposition 5.3]{ziyin2025parameter} we have 
\begin{equation}
\begin{aligned}
S(\theta)&=d_y\Tr[W\Sigma_xW^T]+||U||_F^2\Tr[\Sigma_x]\\&=d_y\sqrt{\frac{\Tr[\Sigma_x]}{\Tr[\Sigma_\epsilon]}}\Tr[S']+\sqrt{\Tr[\Sigma_x]\Tr[\Sigma_\epsilon]}\Tr[\Sigma_\epsilon^{-1}\tilde{U}S'\tilde{U}^T]
\end{aligned}
\end{equation}
This finishes the proof of \eqref{eq:S}. 

Meanwhile we can also calculate
\begin{equation}
U^*,W^*=\arg\min_{U,W}d_y||W\Sigma_x^{1/2}||_F^2+||U||_F^2\Tr[\Sigma_x].
\end{equation}
As $UW\Sigma_x^{1/2}=V\Sigma_x^{1/2}:=\hat{U}\hat{S}\hat{V}$, the minimum is given by
\begin{equation}
U^*=\left(\frac{\Tr\Sigma_x}{d_y}\right)^{1/4}\hat{U}\sqrt{\hat{S}}\hat{U}_1,\ W^*\Sigma_x^{1/2}=\left(\frac{\Tr\Sigma_x}{d_y}\right)^{-1/4}\hat{U}_1\sqrt{\hat{S}}\hat{V}
\end{equation}
and
\begin{equation}
\min S(\theta)=2\sqrt{d_y\Tr\Sigma_x}\Tr\hat{S}
\end{equation}
where $\hat{U}_1$ is an arbitrary orthogonal matrix. 
\end{proof}



\newpage
\section*{NeurIPS Paper Checklist}
\begin{enumerate}

\item {\bf Claims}
    \item[] Question: Do the main claims made in the abstract and introduction accurately reflect the paper's contributions and scope?
    \item[] Answer: \answerYes{}
    \item[] Justification: Our abstract explicitly states that out main contribution is a theoretical framework for understanding the learning dynamics and emergent phenomena of neural networks trained with stochastic gradient descent (SGD) based an entropic loss landscape.
    \item[] Guidelines:
    \begin{itemize}
        \item The answer NA means that the abstract and introduction do not include the claims made in the paper.
        \item The abstract and/or introduction should clearly state the claims made, including the contributions made in the paper and important assumptions and limitations. A No or NA answer to this question will not be perceived well by the reviewers. 
        \item The claims made should match theoretical and experimental results, and reflect how much the results can be expected to generalize to other settings. 
        \item It is fine to include aspirational goals as motivation as long as it is clear that these goals are not attained by the paper. 
    \end{itemize}

\item {\bf Limitations}
    \item[] Question: Does the paper discuss the limitations of the work performed by the authors?
    \item[] Answer: \answerYes{} 
    \item[] Justification: We discussed this in the conclusion section.
    \item[] Guidelines:
    \begin{itemize}
        \item The answer NA means that the paper has no limitation while the answer No means that the paper has limitations, but those are not discussed in the paper. 
        \item The authors are encouraged to create a separate "Limitations" section in their paper.
        \item The paper should point out any strong assumptions and how robust the results are to violations of these assumptions (e.g., independence assumptions, noiseless settings, model well-specification, asymptotic approximations only holding locally). The authors should reflect on how these assumptions might be violated in practice and what the implications would be.
        \item The authors should reflect on the scope of the claims made, e.g., if the approach was only tested on a few datasets or with a few runs. In general, empirical results often depend on implicit assumptions, which should be articulated.
        \item The authors should reflect on the factors that influence the performance of the approach. For example, a facial recognition algorithm may perform poorly when image resolution is low or images are taken in low lighting. Or a speech-to-text system might not be used reliably to provide closed captions for online lectures because it fails to handle technical jargon.
        \item The authors should discuss the computational efficiency of the proposed algorithms and how they scale with dataset size.
        \item If applicable, the authors should discuss possible limitations of their approach to address problems of privacy and fairness.
        \item While the authors might fear that complete honesty about limitations might be used by reviewers as grounds for rejection, a worse outcome might be that reviewers discover limitations that aren't acknowledged in the paper. The authors should use their best judgment and recognize that individual actions in favor of transparency play an important role in developing norms that preserve the integrity of the community. Reviewers will be specifically instructed to not penalize honesty concerning limitations.
    \end{itemize}

\item {\bf Theory assumptions and proofs}
    \item[] Question: For each theoretical result, does the paper provide the full set of assumptions and a complete (and correct) proof?
    \item[] Answer: 
    \answerYes{}
    \item[] Justification: All theorems are stated with the full set of assumptions.
    \item[] Guidelines:
    \begin{itemize}
        \item The answer NA means that the paper does not include theoretical results. 
        \item All the theorems, formulas, and proofs in the paper should be numbered and cross-referenced.
        \item All assumptions should be clearly stated or referenced in the statement of any theorems.
        \item The proofs can either appear in the main paper or the supplemental material, but if they appear in the supplemental material, the authors are encouraged to provide a short proof sketch to provide intuition. 
        \item Inversely, any informal proof provided in the core of the paper should be complemented by formal proofs provided in appendix or supplemental material.
        \item Theorems and Lemmas that the proof relies upon should be properly referenced. 
    \end{itemize}

    \item {\bf Experimental result reproducibility}
    \item[] Question: Does the paper fully disclose all the information needed to reproduce the main experimental results of the paper to the extent that it affects the main claims and/or conclusions of the paper (regardless of whether the code and data are provided or not)?
    \item[] Answer: \answerYes{}
    \item[] Justification: We present all experiment details in Appendix \ref{app:exp}.
    \item[] Guidelines:
    \begin{itemize}
        \item The answer NA means that the paper does not include experiments.
        \item If the paper includes experiments, a No answer to this question will not be perceived well by the reviewers: Making the paper reproducible is important, regardless of whether the code and data are provided or not.
        \item If the contribution is a dataset and/or model, the authors should describe the steps taken to make their results reproducible or verifiable. 
        \item Depending on the contribution, reproducibility can be accomplished in various ways. For example, if the contribution is a novel architecture, describing the architecture fully might suffice, or if the contribution is a specific model and empirical evaluation, it may be necessary to either make it possible for others to replicate the model with the same dataset, or provide access to the model. In general. releasing code and data is often one good way to accomplish this, but reproducibility can also be provided via detailed instructions for how to replicate the results, access to a hosted model (e.g., in the case of a large language model), releasing of a model checkpoint, or other means that are appropriate to the research performed.
        \item While NeurIPS does not require releasing code, the conference does require all submissions to provide some reasonable avenue for reproducibility, which may depend on the nature of the contribution. For example
        \begin{enumerate}
            \item If the contribution is primarily a new algorithm, the paper should make it clear how to reproduce that algorithm.
            \item If the contribution is primarily a new model architecture, the paper should describe the architecture clearly and fully.
            \item If the contribution is a new model (e.g., a large language model), then there should either be a way to access this model for reproducing the results or a way to reproduce the model (e.g., with an open-source dataset or instructions for how to construct the dataset).
            \item We recognize that reproducibility may be tricky in some cases, in which case authors are welcome to describe the particular way they provide for reproducibility. In the case of closed-source models, it may be that access to the model is limited in some way (e.g., to registered users), but it should be possible for other researchers to have some path to reproducing or verifying the results.
        \end{enumerate}
    \end{itemize}

\item {\bf Open access to data and code}
    \item[] Question: Does the paper provide open access to the data and code, with sufficient instructions to faithfully reproduce the main experimental results, as described in supplemental material?
    \item[] Answer: \answerNo{}
    \item[] Justification: The experiments are only for demonstration and are straightforward to reproduce following the description.
    \item[] Guidelines:
    \begin{itemize}
        \item The answer NA means that paper does not include experiments requiring code.
        \item Please see the NeurIPS code and data submission guidelines (\url{https://nips.cc/public/guides/CodeSubmissionPolicy}) for more details.
        \item While we encourage the release of code and data, we understand that this might not be possible, so “No” is an acceptable answer. Papers cannot be rejected simply for not including code, unless this is central to the contribution (e.g., for a new open-source benchmark).
        \item The instructions should contain the exact command and environment needed to run to reproduce the results. See the NeurIPS code and data submission guidelines (\url{https://nips.cc/public/guides/CodeSubmissionPolicy}) for more details.
        \item The authors should provide instructions on data access and preparation, including how to access the raw data, preprocessed data, intermediate data, and generated data, etc.
        \item The authors should provide scripts to reproduce all experimental results for the new proposed method and baselines. If only a subset of experiments are reproducible, they should state which ones are omitted from the script and why.
        \item At submission time, to preserve anonymity, the authors should release anonymized versions (if applicable).
        \item Providing as much information as possible in supplemental material (appended to the paper) is recommended, but including URLs to data and code is permitted.
    \end{itemize}

\item {\bf Experimental setting/details}
    \item[] Question: Does the paper specify all the training and test details (e.g., data splits, hyperparameters, how they were chosen, type of optimizer, etc.) necessary to understand the results?
    \item[] Answer: \answerYes{}
    \item[] Justification:  We disclose all hyperparameters of experiments.
    \item[] Guidelines:
    \begin{itemize}
        \item The answer NA means that the paper does not include experiments.
        \item The experimental setting should be presented in the core of the paper to a level of detail that is necessary to appreciate the results and make sense of them.
        \item The full details can be provided either with the code, in appendix, or as supplemental material.
    \end{itemize}

\item {\bf Experiment statistical significance}
    \item[] Question: Does the paper report error bars suitably and correctly defined or other appropriate information about the statistical significance of the experiments?
    \item[] Answer: \answerYes{}
    \item[] Justification: We plot the standard error as the shaded area.
    \item[] Guidelines:
    \begin{itemize}
        \item The answer NA means that the paper does not include experiments.
        \item The authors should answer "Yes" if the results are accompanied by error bars, confidence intervals, or statistical significance tests, at least for the experiments that support the main claims of the paper.
        \item The factors of variability that the error bars are capturing should be clearly stated (for example, train/test split, initialization, random drawing of some parameter, or overall run with given experimental conditions).
        \item The method for calculating the error bars should be explained (closed form formula, call to a library function, bootstrap, etc.)
        \item The assumptions made should be given (e.g., Normally distributed errors).
        \item It should be clear whether the error bar is the standard deviation or the standard error of the mean.
        \item It is OK to report 1-sigma error bars, but one should state it. The authors should preferably report a 2-sigma error bar than state that they have a 96\% CI, if the hypothesis of Normality of errors is not verified.
        \item For asymmetric distributions, the authors should be careful not to show in tables or figures symmetric error bars that would yield results that are out of range (e.g. negative error rates).
        \item If error bars are reported in tables or plots, The authors should explain in the text how they were calculated and reference the corresponding figures or tables in the text.
    \end{itemize}

\item {\bf Experiments compute resources}
    \item[] Question: For each experiment, does the paper provide sufficient information on the computer resources (type of compute workers, memory, time of execution) needed to reproduce the experiments?
    \item[] Answer: \answerYes{}
    \item[] Justification: Personal computers are sufficient for all our experiments.
    \item[] Guidelines:
    \begin{itemize}
        \item The answer NA means that the paper does not include experiments.
        \item The paper should indicate the type of compute workers CPU or GPU, internal cluster, or cloud provider, including relevant memory and storage.
        \item The paper should provide the amount of compute required for each of the individual experimental runs as well as estimate the total compute. 
        \item The paper should disclose whether the full research project required more compute than the experiments reported in the paper (e.g., preliminary or failed experiments that didn't make it into the paper). 
    \end{itemize}
    
\item {\bf Code of ethics}
    \item[] Question: Does the research conducted in the paper conform, in every respect, with the NeurIPS Code of Ethics \url{https://neurips.cc/public/EthicsGuidelines}?
    \item[] Answer: \answerYes{}
    \item[] Justification: We closely follow the NeurIPS Code of Ethics.
    \item[] Guidelines:
    \begin{itemize}
        \item The answer NA means that the authors have not reviewed the NeurIPS Code of Ethics.
        \item If the authors answer No, they should explain the special circumstances that require a deviation from the Code of Ethics.
        \item The authors should make sure to preserve anonymity (e.g., if there is a special consideration due to laws or regulations in their jurisdiction).
    \end{itemize}

\item {\bf Broader impacts}
    \item[] Question: Does the paper discuss both potential positive societal impacts and negative societal impacts of the work performed?
    \item[] Answer: \answerNA{}
    \item[] Justification: There is no societal impact of the work performed.
    \item[] Guidelines:
    \begin{itemize}
        \item The answer NA means that there is no societal impact of the work performed.
        \item If the authors answer NA or No, they should explain why their work has no societal impact or why the paper does not address societal impact.
        \item Examples of negative societal impacts include potential malicious or unintended uses (e.g., disinformation, generating fake profiles, surveillance), fairness considerations (e.g., deployment of technologies that could make decisions that unfairly impact specific groups), privacy considerations, and security considerations.
        \item The conference expects that many papers will be foundational research and not tied to particular applications, let alone deployments. However, if there is a direct path to any negative applications, the authors should point it out. For example, it is legitimate to point out that an improvement in the quality of generative models could be used to generate deepfakes for disinformation. On the other hand, it is not needed to point out that a generic algorithm for optimizing neural networks could enable people to train models that generate Deepfakes faster.
        \item The authors should consider possible harms that could arise when the technology is being used as intended and functioning correctly, harms that could arise when the technology is being used as intended but gives incorrect results, and harms following from (intentional or unintentional) misuse of the technology.
        \item If there are negative societal impacts, the authors could also discuss possible mitigation strategies (e.g., gated release of models, providing defenses in addition to attacks, mechanisms for monitoring misuse, mechanisms to monitor how a system learns from feedback over time, improving the efficiency and accessibility of ML).
    \end{itemize}
    
\item {\bf Safeguards}
    \item[] Question: Does the paper describe safeguards that have been put in place for responsible release of data or models that have a high risk for misuse (e.g., pretrained language models, image generators, or scraped datasets)?
    \item[] Answer: \answerNA{} 
    \item[] Justification: Our paper poses no such risks.
    \item[] Guidelines:
    \begin{itemize}
        \item The answer NA means that the paper poses no such risks.
        \item Released models that have a high risk for misuse or dual-use should be released with necessary safeguards to allow for controlled use of the model, for example by requiring that users adhere to usage guidelines or restrictions to access the model or implementing safety filters. 
        \item Datasets that have been scraped from the Internet could pose safety risks. The authors should describe how they avoided releasing unsafe images.
        \item We recognize that providing effective safeguards is challenging, and many papers do not require this, but we encourage authors to take this into account and make a best faith effort.
    \end{itemize}

\item {\bf Licenses for existing assets}
    \item[] Question: Are the creators or original owners of assets (e.g., code, data, models), used in the paper, properly credited and are the license and terms of use explicitly mentioned and properly respected?
     \item[] Answer: \answerYes{} 
    \item[] Justification: We use pytorch. We use MNIST and CIFAR datasets, and the publically available pretrained weights from the pytorch website.
    \item[] Guidelines:
    \begin{itemize}
        \item The answer NA means that the paper does not use existing assets.
        \item The authors should cite the original paper that produced the code package or dataset.
        \item The authors should state which version of the asset is used and, if possible, include a URL.
        \item The name of the license (e.g., CC-BY 4.0) should be included for each asset.
        \item For scraped data from a particular source (e.g., website), the copyright and terms of service of that source should be provided.
        \item If assets are released, the license, copyright information, and terms of use in the package should be provided. For popular datasets, \url{paperswithcode.com/datasets} has curated licenses for some datasets. Their licensing guide can help determine the license of a dataset.
        \item For existing datasets that are re-packaged, both the original license and the license of the derived asset (if it has changed) should be provided.
        \item If this information is not available online, the authors are encouraged to reach out to the asset's creators.
    \end{itemize}

\item {\bf New assets}
    \item[] Question: Are new assets introduced in the paper well documented and is the documentation provided alongside the assets?
    \item[] Answer: \answerNA{} 
    \item[] Justification: Our paper does not release new assets.
    \item[] Guidelines:
    \begin{itemize}
        \item The answer NA means that the paper does not release new assets.
        \item Researchers should communicate the details of the dataset/code/model as part of their submissions via structured templates. This includes details about training, license, limitations, etc. 
        \item The paper should discuss whether and how consent was obtained from people whose asset is used.
        \item At submission time, remember to anonymize your assets (if applicable). You can either create an anonymized URL or include an anonymized zip file.
    \end{itemize}

\item {\bf Crowdsourcing and research with human subjects}
    \item[] Question: For crowdsourcing experiments and research with human subjects, does the paper include the full text of instructions given to participants and screenshots, if applicable, as well as details about compensation (if any)? 
    \item[] Answer: \answerNA{} 
    \item[] Justification: Our paper does not release new assets.
    \item[] Guidelines:
    \begin{itemize}
        \item The answer NA means that the paper does not involve crowdsourcing nor research with human subjects.
        \item Including this information in the supplemental material is fine, but if the main contribution of the paper involves human subjects, then as much detail as possible should be included in the main paper. 
        \item According to the NeurIPS Code of Ethics, workers involved in data collection, curation, or other labor should be paid at least the minimum wage in the country of the data collector. 
    \end{itemize}

\item {\bf Institutional review board (IRB) approvals or equivalent for research with human subjects}
    \item[] Question: Does the paper describe potential risks incurred by study participants, whether such risks were disclosed to the subjects, and whether Institutional Review Board (IRB) approvals (or an equivalent approval/review based on the requirements of your country or institution) were obtained?
    \item[] Answer: \answerNA{} 
    \item[] Justification: Our paper does not involve crowdsourcing nor research with human subjects.
    \item[] Guidelines:
    \begin{itemize}
        \item The answer NA means that the paper does not involve crowdsourcing nor research with human subjects.
        \item Depending on the country in which research is conducted, IRB approval (or equivalent) may be required for any human subjects research. If you obtained IRB approval, you should clearly state this in the paper. 
        \item We recognize that the procedures for this may vary significantly between institutions and locations, and we expect authors to adhere to the NeurIPS Code of Ethics and the guidelines for their institution. 
        \item For initial submissions, do not include any information that would break anonymity (if applicable), such as the institution conducting the review.
    \end{itemize}

\item {\bf Declaration of LLM usage}
    \item[] Question: Does the paper describe the usage of LLMs if it is an important, original, or non-standard component of the core methods in this research? Note that if the LLM is used only for writing, editing, or formatting purposes and does not impact the core methodology, scientific rigorousness, or originality of the research, declaration is not required.
    \item[] Answer: \answerNA{}
    \item[] Justification: This research does not involve LLMs as any important, original, or non-standard components.
    \item[] Guidelines:
    \begin{itemize}
        \item The answer NA means that the core method development in this research does not involve LLMs as any important, original, or non-standard components.
        \item Please refer to our LLM policy (\url{https://neurips.cc/Conferences/2025/LLM}) for what should or should not be described.
    \end{itemize}

\end{enumerate}

\end{document}